%% file: main.tex
\documentclass[sigconf]{acmart}

\AtBeginDocument{%
  \providecommand\BibTeX{{%
    \normalfont B\kern-0.5em{\scshape i\kern-0.25em b}\kern-0.8em\TeX}}}

\usepackage{times}
\usepackage{epsfig}
\usepackage{color}
\usepackage{amsmath}
  
\usepackage{amssymb}
\usepackage{verbatim}
\usepackage{color}
\usepackage{hyperref}
\usepackage{amsthm}
\usepackage{algorithm}
\usepackage{algpseudocode}
\usepackage{multirow}
\usepackage{subfigure}

\newtheorem{theorem}{Theorem}[section]

\newtheorem{corollary}[theorem]{Corollary}

\newtheorem{proposition}{Proposition}

\newcommand{\eps}{\varepsilon}

\newcommand{\tr}{\mathsf{Tr}}
\newcommand{\cost}{\mathrm{Cost}}

\newcommand{\GNN}{\mathsf{GNN}}

\newcommand{\real}     {\mathbb{R}}

\newcommand{\norm}[1]  {\| #1 \|}

\newcommand{\cancel}[1]{}

\newcommand{\diag}{\mathsf{diag}}

\renewcommand{\algorithmicrequire}{\textbf{Input:}} % Use Input in the format of Algorithm
\renewcommand{\algorithmicensure}{\textbf{Output:}} % Use Output in the format of Algorithm

\renewcommand{\paragraph}[1]{\medskip \noindent{\bf #1.}}

\copyrightyear{2021}
\acmYear{2021}
\setcopyright{acmlicensed}\acmConference[KDD '21]{Proceedings of the 27th ACM SIGKDD Conference on Knowledge Discovery and Data Mining}{August 14--18, 2021}{Virtual Event, Singapore}
\acmBooktitle{Proceedings of the 27th ACM SIGKDD Conference on Knowledge Discovery and Data Mining (KDD '21), August 14--18, 2021, Virtual Event, Singapore}
\acmPrice{15.00}
\acmDOI{10.1145/3447548.3467256}
\acmISBN{978-1-4503-8332-5/21/08}

\begin{document}
\fancyhead{}
\title{Scaling Up Graph Neural Networks Via Graph Coarsening}

\author{Zengfeng Huang}
\affiliation{School of Data Science\\
\institution{Fudan University}
\country{}}
\email{huangzf@fudan.edu.cn}
\authornote{Corresponding author}

\author{Shengzhong Zhang}
\affiliation{School of Data Science\\
\institution{Fudan University}
\country{}}
\email{szzhang17@fudan.edu.cn}

\author{Chong Xi}
\affiliation{School of Data Science\\
\institution{Fudan University}
\country{}}
\email{cxi19@fudan.edu.cn}

\author{Tang Liu}
\affiliation{
\institution{Fudan University}
\country{}}
\email{cnliutang@gmail.com}

\author{Min Zhou}
\affiliation{
\institution{Huawei Technologies Co. Ltd}
\country{}}
\email{zhoumin27@huawei.com}

\begin{abstract}
	Scalability of graph neural networks remains one of the major challenges in graph machine learning. Since the representation of a node is computed by recursively aggregating and transforming representation vectors of its neighboring nodes from previous layers, the receptive fields grow exponentially, which makes standard stochastic optimization techniques ineffective. Various approaches have been proposed to alleviate this issue, e.g., sampling-based methods and techniques based on pre-computation of graph filters. 
	
	In this paper, we take a different approach and propose to use graph coarsening for scalable training of GNNs, which is generic, extremely simple and has sublinear memory and time costs during training. We present extensive theoretical analysis on the effect of using coarsening operations and provides useful guidance on the choice of coarsening methods. Interestingly, our theoretical analysis shows that coarsening can also be considered as a type of regularization and may improve the generalization.  Finally, empirical results on real world datasets show that, simply applying off-the-shelf coarsening methods, we can reduce the number of nodes by up to a factor of ten without causing a noticeable downgrade in classification accuracy.
	
\end{abstract}





\maketitle

\section{Introduction}
In the recent few years, graph neural network (GNN) has emerged as a major tool for graph machine learning \citep{bruna2013spectral,defferrard2016convolutional,monti2017geometric,kipf2016semi,hamilton2017inductive,velickovic2017graph,liu2020towards,chen2020simple,klicpera2018predict}, which has found numerous applications in scenarios with explicit or implicit graph structures, e.g., \cite{dinella2020hoppity, wei2020lambdanet,paliwal2019reinforced,ying2018graph, zhang2019inductive,li2019encoding,pfaff2020learning}. Despite the tremendous success, the difficulty of scaling up GNNs to large graphs remains one of the main challenges, which limits their usage in large-scale industrial applications. In traditional machine learning settings, the loss function of the model can be decomposed into the individual sample contributions, and hence stochastic optimization techniques working with mini-batches can be employed to tackle training set that is much larger than the GPU memory. However, GNN computes the representation of a node recursively from its neighbors, making the above strategy non-viable, as the loss corresponding to each sample in a $\ell$-layer GNN depends on the subgraph induced by its $\ell$-hop neighborhood, which grows exponentially with $\ell$. Therefore, full-batch gradient descent is often used for training GNNs \cite{kipf2016semi,velickovic2017graph}, but this does not scale to large graphs due to limited GPU memory. 

Recently, a large body of research work studies this issue and various techniques have been proposed to improve the scalability of GNNs.  One prominent direction is to decouple the interdependence between nodes hence reducing the receptive fields. Pioneered by \cite{hamilton2017inductive}, layer-wise sampling combined with mini-batch training has proved to be a highly effective strategy, and since then, several follow-up works try to improve this baseline with optimized sampling process, better stochastic estimations, and other extensions \cite{chen2018fastgcn,chen2018stochastic,zou2019layer,cong2020minimal,ramezani2020gcn}. Another related technique is based on subgraph sampling, which carefully samples a small subgraph in each training iteration and then simply performs full-batch gradient descent on this subgraph \cite{zeng2019graphsaint,chiang2019cluster}.  In practice, performing random sampling from a large graph in each epoch requires many random accesses to the memory, which is not friendly to GPUs \cite{ramezani2020gcn}.

A second approach is largely motivated by \cite{wu2019simplifying}, in which the authors show that removing the nonlinear activations in GCN \cite{kipf2016semi} does not affect the accuracy by much on common benchmarks. The resulting model is simply a linear diffusion process followed by a classifier. Then the diffusion process can be pre-computed and stored, after which the classifier can be trained with naive stochastic optimization. Recently, this idea is extended to more general propagation rules akin to personalized Pagerank, and highly scalable algorithms for pre-computing the propagation process are investigated \cite{bojchevski2020scaling,chen2020scalable}. Although such methods often perform better than sampling-based techniques on popular benchmarks~\cite{chen2020scalable}, they only work for a restricted class of architectures: graph diffusion and nonlinear feature transformation are decoupled, which does not retain the full expressive power of GNNs \cite{xu2018powerful}.

\vspace{-0.3cm}
\paragraph{Our Contributions} In this paper, we investigate a simple and generic approach based on graph coarsening. In a nutshell, our method first applies an appropriate graph coarsening method, e.g., \cite{loukas2019graph}, which outputs a coarse graph with much smaller number of nodes and edges; then trains a GNN on this coarse graph; finally transfers the trained model parameters of this smaller model to the GNN defined on the original graph for making inference. Since, the training is only done on a much smaller graph, the training time and memory cost are \emph{sublinear}, while all previous methods have time complexity at least linear in the number of nodes \cite{chen2020scalable}. Moreover, full-batch gradient descent can be applied, which not only avoids doing random sampling on a large graph repeatedly, but is also much simpler than previous techniques, since any GNN model can be applied directly without changing the code.  Our contributions are summarized as follows.
\vspace{-0.1cm}
\begin{enumerate}
	\item A new method based on graph coarsening for scaling up GNN is proposed, which is generic, extremely simple and has sublinear training time and memory without using sampling.
	\item Extensive theoretical analysis is presented. We analyze the effect of coarsening operations on GNNs quantitatively and provides useful guidance on the choice of coarsening methods. Interestingly, our theoretical analysis shows that coarsening can also be considered as a type of regularization and may improve the generalization, which has been further verified by the empirical results.
	\item Empirical studies on real world datasets show that, simply applying off-the-shelf coarsening methods, we can reduce the number of nodes by up to a factor of ten without causing a noticeable downgrade in classification accuracy.
\end{enumerate}
\vspace{-0.1cm}
We remark that our methods and existing ones mentioned above are \emph{complementary techniques}, and can be easily combined to tackle truly industrial-scale graphs.
\section{Preliminaries}
\subsection{Graph and Matrix Notations}
In this paper, all graphs considered are undirected. A graph with node feature is denoted as $G=(V, E, X)$, where $V$ is the vertex set, $E$ is the edge set, and $X\in \real^{n\times f}$ is the feature matrix (i.e., the $i$-th row of $X$ is the feature vector of node $v_i$). Let $n=|V|$ and $m=|E|$ be the number of vertices and edges respectively. We use $A\in\{0,1\}^{n\times n}$ to denote the adjacency matrix of $G$, i.e., the $(i,j)$-th entry in $A$ is $1$ if and only if their is an edge between $v_i$ and $v_j$. The degree of a node $v_i$, denoted as $d_i$, is the number of edges incident on $v_i$. The degree matrix $D$ is a diagonal matrix and the its $i$-th diagonal entry is $d_i$.

For a $d$-dimensional vector $x$, $\norm{x}_2$ is the Euclidean norm of $x$. We use $x_i$ to denote the $i$th entry of $x$, and $\diag(x)\in\real^{d\times d}$ is a diagonal matrix such that the $i$-th diagonal entry is $x_i$.
We use $A_{i:}$ and $A_{:i}$ to denote the $i$-th row and column of $A$ respectively, and $A_{ij}$ for the $(i,j)$-th entry of $A$.
We use $\norm{A}_2$ to denote the spectral norm of $A$, which is the largest singular value of $A$, and $\norm{A}_F$ for the \emph{Frobenius Norm}, which is $\sqrt{\sum_{i,j}a_{i,j}^2}$. 
%When $A$ is a real symmetric matrix, its spectral norm can be characterized as
%$$\norm{A}_2=\max(\abs{\lambda_1},\abs{\lambda_d})={\max_{y:\norm{y}=1}} \abs{y^TAy}.$$
The trace of a square matrix $A$ is denoted by $\mathsf{Tr(A)}$, which is the sum of the diagonals in $A$. It is well-known that $\tr(A)$ is equal to the sum of its eigenvalues. 
For notational convenience, we always write $A_{P} \triangleq P^TAP$ for any matrix $P$ with the same number of rows as $A$. 

\subsection{Graph Laplacian and Graph Fourier Transformation} 
The Laplacian matrix of a graph $G$ is defined as $L_G=D-A$; when the underling graph $G$ is clear from the context, we omit the subscription and simply write $L$. A key property of $L$ is that its quadratic form measures the "smoothness" of a signal w.r.t. the graph structure, and thus is often used for regularization purposes. More formally, for any vector $x\in \real^n$, it is easy to verify that
\vspace{-0.1cm}
\begin{equation}\label{eqn:LPq}
	x^TLx = \sum_{i,j} A_{ij} (x_i-x_j)^2 = \sum_{(v_i,v_j) \in E} (x_i-x_j)^2.
\end{equation}
Here, $x$ can be viewed as a one-dimensional feature vector and $x^TLx$ measures the smoothness of features across edges. This can be extended to multi-dimensional features. For any matrix $X\in \real^{n\times d}$, where $X_i$ is the feature of the $i$-th node, then we have
\vspace{-0.1cm}
\begin{equation}\label{eqn:trace}
	\sum_{(v_i,v_j) \in E}  \|X_{i:}-X_{j:}\|^2 = \sum_{i,j} A_{ij} \|X_{i:}-X_{j:}\|^2 = 	\mathsf{Tr} (X^TLX).
\end{equation}
In many applications, the symmetric normalized version of $L$, i.e., $D^{-1/2} L D^{-1/2}$,  is the right matrix to consider, which is denoted as $\mathcal{N}$.  Since $\mathcal{N}$ is real symmetric, it can be diagonalized and it is known that all its eigenvalues are in the range $[0,2]$. Let $0=\lambda_1\le \cdots\le \lambda_n\le 2$ be the eigenvalues of $\mathcal{N}$ with corresponding eigenvectors $u_1,\cdots,u_n$ and let $\mathcal{N}=U\Lambda U^T=\sum_{i=1}^n \lambda_i u_i u_i^T$ be the eigen-decomposition. In graph signal processing, given an $n$-dimensional discrete signal $x\in \real^n$, its Graph Fourier Transformation (GFT) is $\hat{x} = U^T x$ \cite{kipf2016semi}. The corresponding eigenvalue of a Fourier mode is the frequency. 
%Once the Fourier modes and their corresponding frequencies are defined, we can define graph filters. For instance, 
From this perspective, the orthogonal projector to the low-frequency eigenspace acts as a low-pass filter, which only retains contents in the lower frequencies; on the other hand, a projector to the high-frequency space is a high-pass filter.
%we can conveniently define graph convolution operations in the frequency domain. 
%More formally, suppose we want to compute the convolution of two signals $y, x$ in the time domain, denoted as $y*x$, we first compute their Fourier transformation $\hat{y}=U^T y$ and $\hat{x} =U^T x$. Since the Fourier transformation of $y*x$ is $\widehat{y*x} = \hat{y}\odot \hat{x}$, where $\odot$ is entry-wise product, $y*x = U \widehat{y*x} = U  \hat{y}\odot \hat{x}$. In convolution neural networks, the convolution filter $y$ is often parameterized. Here we directly parameterize the filter in the frequency domain by $\theta = \hat{y}$ and let $g_{\theta} =\diag(\theta)$. Then the parameterized graph convolution operation can be simplified as
%\begin{equation}
%y*x = U  \hat{y}\odot \hat{x} = U g_{\theta}U^T x,
%\end{equation}
%where $\theta$ are the learnable parameters.

%However, the naive convolution definition requires to compute $U$, which is expensive. To circumvent this issue, various approximation techniques are proposed \cite{hammond2011wavelets}
%Fortunately, we can think of $g_{\theta}$ as a parameterized function of $\Lambda$, i.e., $g_{\theta} = h_{\theta'}(\Lambda)$ for some function $h_{\theta'}(\cdot)$, then
%\begin{equation}
%y*x =  U g_{\theta}U^T x = U h_{\theta'}(\Lambda) U^T x = h_{\theta'}(U\Lambda U^T) x= h_{\theta'}(N)x.
%\end{equation}

\subsection{Graph Neural Networks}
In each layer of a GNN, the representation of a node is computed by recursively aggregating and transforming representation vectors of its neighboring nodes from the last layer. One special case is the Graph Convolutional Network (GCN) \cite{kipf2016semi}, which aims to generalize CNN to graph-structured data. Kipf and Welling \cite{kipf2016semi} define graph convolution (GC) as 
$	Z =  \tilde{D}^{-1/2}\tilde{A}\tilde{D}^{-1/2} X W$,
where $\tilde{A} = A+I$, $\tilde{D}=D+I$, and $W$ is a learnable parameter matrix. GCNs consist of multiple convolution layers of the above form, with each layer followed by a non-linear activation. 
In \cite{klicpera2018predict}, the authors propose APPNP, which uses a propagation rules inspired from personalized Pagerank. More precisely, the APPNP model is defined as follows:
\begin{itemize}
	\item $Z^{(1)} =H\triangleq f(X,W)$ , where $f(X,W)$ is a neural network with parameter set $W$.
	\item $Z^{(k+1)} = (1-\beta) \tilde{D}^{-1/2}\tilde{A} \tilde{D}^{-1/2} Z^{(k)} + \beta H$, where $\beta \in(0,1]$ is a hyperparameter.
\end{itemize}

\section{Our Method}
\subsection{Graph Coarsening}\label{sec:partitionmatrix}
Given a graph $G=(V,E,X)$, the coarse graph is a smaller weighted graph $G' = (V', E', X', W)$ with edge weights $W$. Denote $n' \triangleq |V'|$ and $m'\triangleq |E'|$. $G'$ is obtained from the original graph by first computing a partition $P =\{C_1, C_2, \cdots, C_{n'}\}$ of $V$, i.e., the clusters $C_1\cdots C_{n'}$ are disjoint and cover all the nodes in $V$. Each cluster $C_i$ corresponds to a “super-node” in $G'$ and the “super-edge” connecting the super-nodes $C_i,C_j$ has weight equal to the total number of edges connecting nodes in $C_i$ to  $C_j$:
$W_{ij} = \sum_{u\in C_i, v\in C_j} A_{ij}.$

The partition can be represented by a matrix $\hat{P}\in\{0,1\}^{n\times k}$, with $\hat{P}_{ij} = 1$ if and only if vertex $i$ belongs to cluster $C_j$. So, each row of $P$ contains exactly one nonzero entry and columns of $P$ are pairwise orthogonal. Then $W=  A_{\hat{P}} \triangleq \hat{P}^T A \hat{P}$ and $A_{\hat{P}}$ is identified as the adjacency matrix of $G'$. Similarly, $D_{\hat{P}} \triangleq  \hat{P}^T D \hat{P}$ is the degree matrix of $G'$. Note that the number of edges in the coarse graph is also significantly smaller than $m$, as each super-edge combines many edges in the original graph. It means that the number of non-zero entries in the adjacency matrix $A_{\hat{P}}$ is much smaller than $A$.

Let $c_j, j=1,\cdots n'$ be the number of vertices in $C_j$, and $C\triangleq \diag(c_1,\cdots, c_k)$. The normalized version of $\hat{P}$ is ${P} \triangleq \hat{P}C^{-1/2}$, i.e., $P_{ij} = 1/\sqrt{c_j}$ if $v_i\in C_j$ and $0$ otherwise. It is easy to verify that ${P}$ has orthonormal columns, and thus ${P}^T{P} = I$. We use $\mathcal{P}$ to denote the set of all normalized partition matrices.
\vspace{0.5cm}
\subsection{Our Method}\label{sec:ourmethod}
\paragraph{The generic algorithm}
We mainly focus on the semi-supervised node classification setting, where we are given an attributed graph $G=(V,E,X)$ and labels for a small subset of nodes. Assume the number of classes is $l$. We use $Y\in\{0,1\}^{n\times l} $ to represent the label information: if $v_i$ is labeled, then $Y_{i:}$ is the corresponding one-hot indicator vector, otherwise $Y_{i:}=0$.  We use $\GNN_G(W)$ to denote the GNN model based on $G$. Given a loss function $\ell$, e.g., cross entropy, the loss of the model is denoted as $\ell(\GNN_G(W), Y)$. The training algorithm is to minimize the loss w.r.t. $W$. The time and memory costs of training are proportional to the size of $G$. To improve the computational costs, we first compute a coarse approximation of $G$, denoted as $G'$, via graph coarsening described above, then minimize the loss $\ell(\GNN_{G'}(W), Y')$ w.r.t. $W$. The optimal parameter matrix $W^*$ is then used in $\GNN_G()$ for prediction.

In the coarse graph, each node is a super-node corresponding to a cluster of nodes in the original graph. The feature vector of each super-node is the mean of the feature vectors of all nodes in the cluster, i.e., $X'=P^TX$. We set the label of each super-node similarly, i.e., $P^TY$. However, it is possible that the super-node contains nodes from more than one class. For this case, we pick the dominating label, i.e., apply a row-wise argmax operation on $P^TY$. In our experiments, we find that discarding such super-nodes with mixed labels often benefits the accuracy.  However, in general, more sophisticated aggregation schemes can be applied to suit the application at hand. See Algorithm~\ref{alg:generic} for the description of our framework. We remark that graph coarsening can be efficiently pre-computed on CPUs, where the main memory size could be much larger than GPUs.
\begin{algorithm}[h]
	\caption{Training GNN with Graph Coarsening }\label{alg:generic}
	\begin{algorithmic}[1]
		\Require 
		$G=(V,E,X)$; Labels $Y$;  Model $\GNN_G(W)$; Loss $\ell$; 
		\Ensure 
		Output trained weight matrix $W^*$
		\State Apply a graph coarsening algorithm on $G$, and output a normalized partition matrix $P$. 
		\State Construct the coarse graph $G'$ using P;
		\State Compute the feature matrix of $G'$ by $X'=P^TX$
		\State Compute the labels of $G'$ by $Y' = \arg\max (P^TY)$
		\State Train parameter $W$ to minimize the loss $\ell(\GNN_{G'}(W),Y')$ to obtain a optimal weight matrix $W^*$\\
		\Return $W^*$;
	\end{algorithmic}
\end{algorithm}
The coarse graph $G'$ is weighted and the number of nodes in each super-node may vary significantly. Thus, when constructing the smaller model $\GNN_{G'}(W)$, we sometimes need to revise the propagation scheme. Next, we give a slightly more general GC, which is motivated from our theoretical analysis in Section~\ref{sec:theory}.

\paragraph{Graph convolution on the coarse graph} 
We define the convolution operation on $G'$ as
$$Z= ({D}_{\hat{P}}+C)^{-1/2}({A}_{\hat{P}} + C)({D}_{\hat{P}}+C)^{-1/2} X' W.$$
Here we add $C$ instead of $I$ as in \cite{kipf2016semi} to reflect the relative size of each super-node, for which we will give a theoretical justification in Section~\ref{sec:theory}. Also, this definition includes the standard GC as a special case, i.e., when there is no coarsening, then $C=I$.
By definitions of $P$ and $C$, we have
$$\tilde{A}_{P} \triangleq P^T\tilde{A}P = P^T(A+I) P = A_P + I = C^{-1/2}A_{\hat{P}}C^{-1/2} +I,$$
$$\tilde{D}_{P} \triangleq P^T\tilde{D}P =P^T(D+I) P = D_P + I = C^{-1/2}D_{\hat{P}}C^{-1/2} +I.$$
Since $D_{\hat{P}}$ is diagonal, $\tilde{D}_{P} $ is further simplified to $C^{-1}D_{\hat{P}}+I= C^{-1}(D_{\hat{P}}+C)$. Then the coarse graph convolution is equivalent to
\begin{equation}\label{eqn:GCcoarse}
	Z= \tilde{D}_{{P}}^{-1/2} \tilde{A}_{P} \tilde{D}_{{P}}^{-1/2}X' W,
\end{equation}
which looks more similar to the standard GC.

\section{Theoretical Foundations}\label{sec:theory}
Note that, when $\beta\rightarrow 0$, the propagation step of APPNP is the same as GCN. So one can think of GCN as a model which stacks multiple single-step APPNP models, interlaced by non-linear activations. In this section, we provide rigorous analysis on how APPNP behaves on the coarse graph, present theoretical guarantees on the approximation errors of different coarsening methods, and make interesting connections to existing graph coarsening schemes.
We first provide a variational characterization of APPNP, from which we derive APPNP on the coarse graph.
\subsection{A Characterization of APPNP}
Let $Z^{(t)}$ be the output of the $t$-th layer in APPNP.
It can be shown that $Z^{t}$ converges to the solution to a linear system, see e.g., \cite{klicpera2018predict,zhou2004learning}.
\begin{proposition}
	$Z^{(\infty)}$ is the solution to the linear system
	\begin{equation}\label{eqn:pagerankLS}
		\left(I-(1-\beta) \tilde{D}^{-1/2}\tilde{A} \tilde{D}^{-1/2}\right)Z = \beta H\triangleq f(X,W). 
	\end{equation}
\end{proposition}
It is known that the above linear system is non-singular (in fact positive definite) when $\beta$ is strictly positive  \cite{chung1997spectral}, and thus the solution exists and is unique. 
%For small graphs, one can also solve the system analytically and train the parameters in $H$ directly, which is called PPNP in \cite{klicpera2018predict}.
It is a standard fact in numerical optimization that the solution to such a linear system is the optima of some convex quadratic optimization problem.
\begin{proposition}\label{lem:PPNPvr}
	Let $Y^*$ be the optima of the following quadratic optimization problem:
	\begin{equation}
		\min_{Y\in\real^{n\times h}}  (1-\beta) \tr \left(Y^TLY\right) + \beta \| \tilde{D}^{1/2}Y - H\|_F^2. \label{eqn:PPNPvr}
	\end{equation}
	Then $Z^* = \tilde{D}^{1/2} Y^*$ is the unique solution to \eqref{eqn:pagerankLS}.
\end{proposition}

%\begin{proof}
%	We calculate its gradient and set it to zero:
%	\begin{equation*}
%	(1-\beta) (D-A)Y^* +\beta (D+I) Y^*-\beta\tilde{D}^{1/2}f(X,W) =0,
%	\end{equation*}
%	which implies
%	$$(D+I-(1-\beta)(A+I) )Y^*=\beta \tilde{D}^{1/2}f(X,W).$$
%	Multiplying both sides by $\tilde{D}^{-1/2}$, the linear system becomes
%	\begin{equation*}
%	(D+I)^{-1/2}(D+I-(1-\beta)(A+I) )Y^*=\beta f(X,W),
%	\end{equation*}
%	or equivalently
%	$$(D+I)^{1/2}(I-(1-\beta)M' )Y^*=\beta f(X,W),$$
%	where $M' = \tilde{D}^{-1}\tilde{A}$. Reparameterize $Z^*= \tilde{D}^{1/2}Y^*$, we have
%	\begin{equation*}
%	\left(I-(1-\beta)\tilde{D}^{1/2} M' \tilde{D}^{-1/2}\right)Z^* = \beta f(X,W),
%	\end{equation*}
% which is equivalent to
%	$$	\left(I-(1-\beta)\tilde{D}^{-1/2} \tilde{A} \tilde{D}^{-1/2}\right)Z^* =\beta f(X,W).$$
%\end{proof}

Let $L'$ be the Laplacian of the coarse graph. APPNP on the coarse graph corresponds to an optimization problem of the same form except $L$ is replaced by $L'$. With this perspective, we can quantitatively analyze the effect of replacing $L$ with $L'$ in APPNP.
Of course the quadratic variational representation is not unique, and similar formulations have been used to derive label and feature propagation schemes \cite{zhou2004learning,eliav2018bootstrapped,zhu2021interpreting}. 

In \eqref{eqn:PPNPvr}, the optimization problem is unconstrained. To motivate graph coarsening, we generalize it to the constrained case, where we require $Y\in \mathcal{C}\subseteq \real^{n\times h}$ for some constraint set $\mathcal{C}$, i.e.,
\begin{equation}
	\min_{Y\in\mathcal{C}}  (1-\beta) \tr \left(Y^TLY\right) + \beta \| \tilde{D}^{1/2}Y - H\|_F^2. \label{eqn:PPNPvrc}
\end{equation}
We will show that applying graph coarsening is roughly equivalent to putting a special constraint $\mathcal{C}$ on APPNP. Therefore, coarsening can also be considered as a type of regularization and may improve the generalization, which is verified by our empirical results.

\paragraph{Possible Choices of $\mathcal{C}$} The canonical example of $\mathcal{C}$ is a set of matrices whose columns are within some $k$-dimensional subspace with $k< n$. More precisely, let $V\in \real^{n\times k}$ be an orthonormal basis of the $k$-dimensional subspace, then $C = \{Y: Y=VR, \textrm{for some } R\in \real^{k\times h} \}$. Different choices of $\mathcal{C}$'s give rise to different variants of APPNP, e.g., one could encode sparsity, rank, and general norm constraints in $\mathcal{C}$, which may be highly useful depending on the tasks and datasets at hand. For the graph coarsening purpose, we will only focus on the case where $\mathcal{C}$ is a subspace. Nevertheless, being subspaces has already included many interesting special cases. For instance, when $\mathcal{C}$ is the eigenspace of the normalized Laplacian $\mathcal{N}$ corresponding to small eigenvalues, then it acts as a low-pass filter; on the other hand, when $\mathcal{C}$ consists of eigenvectors with high eigenvalues, then it is a high-pass filter.

\subsection{Subspace Constraints and Dimensionality Reduction}
In this subsection, we show that subspace constraints will benefit computation, as we essentially only need to solve a lower-dimensional problem. 
From now on, $\mathcal{C}$ is always a linear subspace of dimension $k$, and let $V\in \real^{n\times k}$ be an orthonormal basis of $\mathcal{C}$. As a result, \eqref{eqn:PPNPvrc} can be rewritten as
\begin{equation}
	\min_{Y: Y=VR \textrm{ for some } R\in \real^{k\times h} }  (1-\beta) \tr \left(Y^TLY\right) + \beta \| \tilde{D}^{1/2}Y - H\|_F^2. \label{eqn:PPNPsubspace}
\end{equation}
Thus, we only need to solve a lower-dimensional problem
\begin{equation}
	R^* = \arg\min_{R\in\real^{k\times h}}  (1-\beta) \tr \left(R^T V^TL VR\right) + \beta \| \tilde{D}^{1/2}VR - H\|_F^2. \label{eqn:reducedPPNP}
\end{equation}
The optima of \eqref{eqn:PPNPsubspace} can be recover via $Y^* = VR^*$.
Let $L_V = V^TLV$, which is an $k\times k$ matrix and thus much smaller than $L$, similarly let $A_V=V^TAV$, $\tilde{A}_V = V^T\tilde{A}V = A_V+I$, $D_V=V^TDV$ and $\tilde{D}_V =V^T \tilde{D} V =D_V+I$.

\begin{theorem}
	Let $R^*$ be the optima of the quadratic optimization problem \eqref{eqn:reducedPPNP}.
	Then $Z^* = \tilde{D}_V^{1/2} R^*$ is the unique solution to the linear system
	\begin{equation*}
		\left(I-(1-\beta)\tilde{D}_V^{-1/2} \tilde{A}_V \tilde{D}_V^{-1/2}\right)Z = \beta \tilde{D}_V^{-1/2}V^T\tilde{D}^{1/2}F.
	\end{equation*}
\end{theorem}
\begin{proof}
	By taking the gradient of \eqref{eqn:reducedPPNP} and set it to $0$, we have
	\begin{align*}
		(1-\beta) (D_V-A_V)R^* +\beta (D_V+I) R^*-\beta V^T\tilde{D}^{1/2}F =0.
	\end{align*}
	By rearranging the terms, it implies
	$$\left( D_V+I - (1-\beta)(A_V+I) \right)R^* = \beta V^T\tilde{D}^{1/2}F$$
	$$\implies \tilde{D}_V^{1/2}(\tilde{D}_V^{1/2}-(1-\beta) \tilde{D}_V^{-1/2}\tilde{A}_V )R^*=\beta V^T\tilde{D}^{1/2}F.$$
	 Multiply both sides by $\tilde{D}_V^{-1/2}$ and reparameterize $Z^*=\tilde{D}_V^{1/2} R^*$,
	$$	\left(I-(1-\beta)\tilde{D}_V^{-1/2} \tilde{A}_V \tilde{D}_V^{-1/2}\right)Z^* = \beta \tilde{D}_V^{-1/2}V^T\tilde{D}^{1/2}F,$$
	which proves the lemma.
\end{proof}
One should see the resemblance between the above equation and \eqref{eqn:pagerankLS}, and thus we may approximately solve $Z^*$ using the same propagation rule. 
\begin{corollary}\label{lem:PPNPreduce}
	Consider the propagation rule:
	\begin{itemize}
		\item $Z^{(1)} =H'\triangleq \tilde{D}_V^{-1/2}V^T\tilde{D}^{1/2}H$ , 
		\item $Z^{(k+1)} = (1-\beta) \tilde{D}_V^{-1/2}\tilde{A}_V \tilde{D}_V^{-1/2} Z^{(k)} + \beta H'$.
	\end{itemize}
	Then, $Z^{t}$ converges to $Z^*$.
\end{corollary}

%\end{proof}

This is almost the same as APPNP, but now the dimension, i.e., the size of the symmetric propagation matrix $\tilde{D}_V^{-1/2}\tilde{A}_V \tilde{D}_V^{-1/2}$ is $k$ by $k$, which is smaller than that in the original APPNP.

Unfortunately, now the time to compute the propagation matrix, $\tilde{D}_V^{-1/2}\tilde{A}_V \tilde{D}_V^{-1/2}$, is $O(nk^2)$ which is expensive for moderately large $k$. Note the original propagation matrix $\tilde{D}^{-1/2}\tilde{A} \tilde{D}^{-1/2}$ can be computed in time $O(m)$, and for sparse graph this is only $O(n)$. Moreover, $\tilde{D}_V^{-1/2}\tilde{A}_V \tilde{D}_V^{-1/2}$ is a dense matrix, which requires $O(k^2)$ space to store and in each propagation, the time complexity is $O(k^2 h)$, where $h$ is the size of feature vectors in $Z^{(t)}$. In comparison, for sparse graphs, $\tilde{D}^{-1/2}\tilde{A} \tilde{D}^{-1/2}$ only requires $O(m)$ space and each propagation takes $O(mh)$ time. Therefore, unless the reduction ratio is extremely high, say reduce from $10^6$ to $10^3$, the computational costs and space usage could increase significantly, which defeats the purpose of graph coarsening in the first place.

\subsection{Sparse Projections and Graph Coarsening}
To overcome the above issues, we restrict the orthonormal matrix $V$ to be sparse. In this paper, we will only consider the family of normalized partition matrices of size $n\times k$ (see Section~\ref{sec:partitionmatrix} for the definitions), denoted as $\mathcal{P}$. 
%A $k$-partition of the vertex set maps each vertex in the graph to exactly one of the $k$ clusters, which can be represented as a matrix $\hat{P}\in\{0,1\}^{n\times k}$, with $\hat{P}_{ij} = 1$ if and only if vertex $i$ belongs to cluster $j$. So, each row of $P$ contains exactly one nonzero entry and columns of $P$ are pairwise orthogonal. Let $c_j, j=1,\cdots k,$ be the number of vertices in cluster $j$, and $C=\diag(c_1,\cdots, c_k)$. The normalized version of $\hat{P}$ is ${P} = \hat{P}C^{-1/2}$. It is easy to verify that ${P}$ has orthonormal columns, and thus ${P}^T{P} = I$.
Given a target constraint subspace $\mathcal{C}$ and its orthonormal basis $V$, we will first find a matrix  ${P}\in \mathcal{P}$ that is close to $V$ and then replace $V$ with ${P}$ in \eqref{eqn:reducedPPNP}. Since $P$ is also an orthonormal matrix, we can apply Corollary \ref{lem:PPNPreduce} directly. 
Therefore, for this surrogate quadratic objective, the propagation rule become 

1) $Z^{(1)} =H:=  \tilde{D}_{{P}}^{-1/2}P^T\tilde{D}^{1/2}F$ ,

2) $Z^{(k+1)} = (1-\beta) \tilde{D}_{{P}}^{-1/2}\tilde{A}_{{P}} \tilde{D}_{{P}}^{-1/2} Z^{(k)} + \beta H$.
%\begin{itemize}
%	\item $Z^{(1)} =H:=  \tilde{D}_{{P}}^{-1/2}P^T\tilde{D}^{1/2}F$ , 
%	\item $Z^{(k+1)} = (1-\beta) \tilde{D}_{{P}}^{-1/2}\tilde{A}_{{P}} \tilde{D}_{{P}}^{-1/2} Z^{(k)} + \beta H$.
%\end{itemize}

The above propagation converges to some $R$ close to $R^*$ \eqref{eqn:reducedPPNP}, as long as $P\approx V$. How to find such a $P$ will be discussed in the following subsections. Recall that $\tilde{A}_{{P}} = A_{{P}}+I ={P}^TA{P} +I$; and $\tilde{D}_{P}=P^TDP+I$ is still a diagonal matrix. Note when $P=I$, i.e., there is no coarsening, this recovers APPNP. Moreover, the propagation matrix $\tilde{D}_{{P}}^{-1/2}\tilde{A}_{{P}} \tilde{D}_{{P}}^{-1/2} $ is exactly the graph convolution we defined for coarse graph in Section~\ref{sec:ourmethod}.
Now since $P$ contains one non-zero entry per row, the time complexity to compute $\tilde{D}_{{P}}^{-1/2}\tilde{A}_{{P}} \tilde{D}_{{P}}^{-1/2} $ is $O(m)$, which is $O(n)$ for sparse graph. 
The number of non-zero entries in $\tilde{D}_{{P}}^{-1/2}\tilde{A}_{{P}} \tilde{D}_{{P}}^{-1/2} $ is $m'$, i.e., the number of super-edges in $G'$. Then the space to store it is $O(m')$ and the time complexity to compute each propagation is $O(m'h)$. Thus, the time and space complexity in the forward pass are improved by a factor of $\frac{m}{m'}$ over the original graph, and note $m'$ could be much smaller than $m$ as each super-edge corresponds to many edges in the original graph.
More importantly, the number of nodes is reduced from $n$ to $n'$. So, the space and time complexity in backpropagation are improved by a factor of $\frac{n}{n'}$.

\subsection{Nuclear Norm Error, $k$-Means, and Spectral Clustering}
From the above discussion, the main question left is how to efficiently compute a partition matrix $P$ which is a good approximation to the target orthonormal matrix $V$.
In this subsection, we provide suitable metrics to quantify the approximation error and give efficient and effective approximation algorithms. 

Our goal is to find a matrix $P\in\mathcal{P}$ whose column space is close to the space spanned by $V$. Since both $P$ and $V$ are orthonormal, so if $P$ is close to $V$, then $P^TV$ should be close to identity. Hence, one natural error metric is the distance between $P^TV$  and $I$. Since $P^TV$ is not symmetric in general, it is more convenient to measure the distance between $V^TPP^TV$ and $I$, or $\norm{V^TPP^TV-I}$ for some matrix norm $\norm{\cdot}$. 
We next show that, when the matrix norm is chosen to be the nuclear norm (denoted as $\norm{\cdot}_1$), i.e., the sum of singular values, the problem is equivalent to $k$-means clustering. 

\begin{theorem}\label{lem:nucleartokmeans}
	Let $S=\{v_1,\cdots, v_n\}$ be a set of $n$ points, where $v_i$ is the $i$-th row of $V$. Let $\cost(P)$ be the $k$-means cost of the partition induced by $P$ with respect to $S$. Then we have $\norm{I-V^TPP^TV}_1=\cost(P)$ for all $P\in \mathcal{P}$.  
\end{theorem}

\begin{proof}
	First observe that the matrix $I-V^TPP^TV$ is positive semidefinite, and therefore 
	\begin{equation}\label{eqn:nucleartotrace}
		\norm{I-V^TPP^TV}_1 = \tr\left(I-V^TPP^TV\right).
	\end{equation}
	Moreover, 
	\begin{align*}
		\tr\left(I-V^TPP^TV\right) &= 	\tr\left(V^TV-V^TPP^TV\right)\\
		&=  	\tr\left(V^TV-2V^TPP^TV + V^TPP^TV\right)\\
		&=\tr\left(V^TV-2V^TPP^TV + V^TPP^TPP^TV\right)\\
		&=\tr\left((PP^TV -V)^T (PP^TV -V)\right)\\
		& = \norm{PP^TV -V}_F^2,
	\end{align*}
	where in the last equality, we use the fact that $\norm{A}_F^2 = \tr(A^TA)$ for any $A$.
	Together with \eqref{eqn:nucleartotrace}, we have 
	\begin{equation}\label{eqn:nucleartoFrobenius}
		\norm{I-V^TPP^TV}_1 =\norm{PP^TV -V}_F^2.
	\end{equation}
	The r.h.s. of \eqref{eqn:nucleartoFrobenius} is exactly the $k$-means cost of the partition induced by $P$. To see this, let $C_1,\cdots, C_k$ be the clusters of points in this partition, i.e., $v_i\in C_j$ iff $P_{ij} \neq 0$, then the corresponding $k$-means cost of this partition is 
	\begin{equation}\label{eqn:k-meanscost}
		\cost(P) =  \sum_{j=1}^k \sum_{v\in C_j } \norm{v - g_j}_2^2,
	\end{equation}
	where $g_j$ is the centroid of the $j$-th cluster. Recall the definition of $\hat{P}$ (with $P=\hat{P}C^{-1/2}$), which is the unnormalized partition matrix.  Then $g_j = \frac{1}{|C_j|} \sum_{v\in C_j } v = \frac{1}{c_j}\hat{P}_{:j}^T V.$ Therefore,
	\begin{align*}
		\cost(P) &=  \sum_{j=1}^k \sum_{v\in C_j } \norm{v - \frac{1}{c_j}\hat{P}_{:j}^T V}_2^2 =  \sum_{i=1}^n \norm{v_i - \frac{1}{c_j}\hat{P}_{i:} \hat{P}^T V}_2^2\\
		& = \norm{PP^TV -V}_F^2.
	\end{align*}
	By \eqref{eqn:nucleartoFrobenius}, we have $\cost(P)  = \norm{I-V^TPP^TV}_1$ for all normalized partition matrix $P\in\mathcal{P}$, which proves the Lemma.
\end{proof}

We have the following simple corollary.
\begin{corollary} 
	$P^* = \arg\min_{P\in \mathcal{P}} \norm{I-V^TPP^TV}_1$ if and only if
	the partition induced by $P^*$ has optimal $k$-means cost w.r.t. $S$. 
\end{corollary}

\paragraph{Connection to Spectral Clustering} From the above corollary, to obtain a good approximation $P$ in terms of nuclear norm, it is equivalent to solve the $k$-means problem w.r.t. $V$. 
Note that when $V$ consists of the $k$ eigenvectors of the normalized Laplacian $\mathcal{N}$ with lowest eigenvalues,  then applying $k$-means to $V$ is the \emph{spectral clustering} algorithm. Thus, in this paper, we provide an alternative explanation of the role of $k$-means in spectral clustering algorithms.

For sparse graphs, the time to compute the $k$ lowest eigenvectors will be dominated by the complexity of $k$-means computation. In the worst case, the $k$-means problem is known to be NP-hard, and approximation algorithms are used in practice, e.g., Lloyd's algorithm \cite{lloyd1982least}, which takes $O(nkd)$ time per iteration, where $d$ is the dimension of each points. For spectral clustering $d=k$. Therefore, spectral clustering does not scale well to large graphs for our application, since $k$, the number of clusters, will be quite large compared to typical graph clustering scenarios. We next investigate a relaxed error norm, and make a connection to a recent work of Loukas \cite{loukas2019graph} on graph coarsening.

\subsection{Spectral Norm Error}
In the above subsection, we measure the error of $P$ w.r.t. $V$ by $\norm{I-V^TPP^TV}_1$, which is the sum of singular values; next we relax this to the spectral norm $\norm{I-V^TPP^TV}_2$, i.e., the maximum singular value. We have
\begin{align*}
	\norm{I-V^TPP^TV}_2 &= \max_{x:\|x\|_2=1} \left|x^T (V^TV-V^TPP^TV) x \right|\\
	&=  \max_{x:\|x\|_2=1} \left| x^TV^TVx-  x^TV^TPP^TPP^TVx \right|\\
	&=  \max_{x:\|x\|_2=1} \left| \norm{Vx}_2^2 - \norm{PP^TVx}_2^2 \right|\\
\end{align*}
Note that $y=Vx$ has norm $1$ for any unit-norm $x$, and thus $\{y: y=Vx, \forall $x$ \textrm{ s.t. } \|x\|_2=1 \}$ is the set of all unit vector in the subspace spanned by $V$, i.e., $\mathcal{C}$. Thus we have
\begin{align}
	\norm{I-V^TPP^TV}_2 &=\max_{y\in \mathcal{C}, \|y\|=1}  \left| \norm{y}_2^2 - \norm{PP^Ty}_2^2 \right|  \nonumber\\
	&=\max_{y\in \mathcal{C}} \frac{\left| \norm{y}_2^2 - \norm{PP^Ty}_2^2 \right|}{\|y\|_2^2} \nonumber\\
	&= \max_{y\in \mathcal{C}} \frac{ \norm{y - PP^Ty}_2^2 }{\|y\|_2^2}, \label{eqn:spectralerror}
\end{align}
where the last equality is from Pythagorean theorem (since $PP^T$ is an orthogonal projection). This is essentially equivalent to the Grassmannian distance between two subspaces, which is defined as $\|PP^T-VV^T\|_2$. The equivalence proof is nontrivial and can be found in the book \cite{kato2013perturbation} (Theorem 6.34).

The above error measure is independent on the underlying graph. In many graph applications, it is often more suitable to use a generalized Euclidean norm $\norm{\cdot}_L$, i.e., $\|x\|_L = \sqrt{x^TLx}$, where $L$ is the Laplacian of the graph. Using this generalized norm in \eqref{eqn:spectralerror}, we will consider the following graph dependent error metric:
\begin{equation}
	\max_{y\in \mathcal{C}} \frac{\norm{y - PP^Ty}_L^2 }{\|y\|_L^2}. \label{eqn:L-spectralerror}
\end{equation}

It is still difficult to efficiently compute an partition matrix $P$ that minimize the above objective. Fortunately, this objective has been studied in \cite{loukas2019graph} recently (see Definition 11 in \cite{loukas2019graph}), and the author proposed efficient approximation algorithms for the case when $V$ is the first $k$ eigenvectors. Moreover, several effective heuristics are discussed and tested empirically on real world datasets.  

In our experiments, the coarsening algorithms from \cite{loukas2019graph}, which aim to minimize \eqref{eqn:L-spectralerror}, perform better than spectral clustering. We believe this is mainly due to the generalized Euclidean norm used. Next, we provide a theoretical explanation on this.
\begin{theorem}
	Suppose $\max_{y\in \mathcal{C}} \frac{\norm{y - PP^Ty}_L }{\|y\|_L} \le \eps<1$, then we have for any $y\in \mathcal{C}$, there exists $x\in \real^k$ such that
	\begin{equation*}
		\left| y^TLy - x^TP^TLPx \right| \le 3\eps \|y\|_L^2.
	\end{equation*}
\end{theorem}
\begin{proof}
	Given $y$, we simply set $x=P^Ty$. Then,
	\begin{align*}
		\left| \sqrt{y^TLy} - \sqrt{x^TP^TLPx} \right| &= \left| \sqrt{y^TLy} - \sqrt{y^TPP^TLPP^Ty} \right|\\
		&= \left| \|L^{1/2} y\|_2 - \|L^{1/2}PP^Ty\|_2 \right|\\
		&\le \|L^{1/2} (y-PP^T y)\|_2 \quad\textrm{Triangle inequality}\\
		& = \|y-PP^Ty\|_L\\
		& \le \eps \sqrt{y^TLy} \quad \textrm{By assumption}
	\end{align*}
	Equivalently, $(1-\eps)\|y\|_L\le \sqrt{x^TP^TLPx}  \le (1+\eps) \|y\|_L$, which implies
	$	(1-\eps)^2\|y\|_L^2\le {x^TP^TLPx}  \le (1+\eps)^2 \|y\|_L^2$.
	Since $(1-\eps)^2 =1-2\eps+\eps^2 \ge 1-2\eps$ and $(1+\eps)^2 = 1+2\eps+\eps^2\le 1+3\eps$, the theorem follows from the above inequalities.
\end{proof}
Similarly, if we have $\max_{y\in \mathrm{span}(P)} \frac{\norm{y - VV^Ty}_L }{\|y\|_L} \le \eps<1$, we can also prove that, for any $x\in\real^k$, there exists $y\in \mathcal{C}$ such that
\begin{equation*}
	\left| y^TLy - x^TP^TLPx \right| \le 3\eps \|Px\|_L^2.
\end{equation*} 
For graph coarsening, we essentially replace the graph regularization term $E(y) = y^TLy, y\in\mathcal{C}$ in \eqref{eqn:PPNPvrc} by $E'(x) = x^TP^T LP^Lx, x\in \real^k$.
So if the two conditions holds simultaneously, this replacement does not change the optimization problem by much, and the resulting embedding should be similar, which is qualitatively verified in the experiments. 
\section{Related Work}
To overcome the scalability issue of training GNNs. Layer-wise sampling combined with mini-batch training has been extensively studied \cite{hamilton2017inductive, chen2018fastgcn,chen2018stochastic,zou2019layer,cong2020minimal,ramezani2020gcn}. Subgraph sampling for scaling up GNNs, which sample a small subgraph in each training iteration and perform full-batch training on this subgraph, is also explored recently \cite{zeng2019graphsaint,chiang2019cluster}. The authors in \cite{ramezani2020gcn} study the problem of how to reduce the sampling frequency in aforementioned sub-sampling approaches. Edge sampling is also used as effective tool for tackling oversmoothing \cite{rong2019dropedge}.
Another approach focuses on how to simplify the models without sacrificing, in particular, to decouple the graph diffusion process from the feature transformation. In this way, the diffusion process can be pre-computed and stored, after which the classifier can be trained with naive stochastic optimization\cite{bojchevski2020scaling,chen2020scalable,wu2019simplifying}. \cite{rossi2020sign} propose a method to pre-compute and store graph convolutional filters of different size.
Graph reduction techniques have been used to speed up combinatorial problems \cite{moitra2009approximation, englert2014vertex}. Graph reduction with spectral approximation guarantees are studied in \cite{loukas2019graph,jin2020graph,li2018spectral}. Recently, graph coarsening has been applied to speedup graph embedding algorithms \cite{fahrbach2020faster,deng2019graphzoom,liang2018mile}. As far as we are aware, this is the first work applying graph coarsening to speedup the training of GNNs in the semi-supervised setting.

\input{exp}

\section{Conclusion}
In this paper, we propose a different approach, which use graph coarsening, for scalable training of GNNs. Our method is generic, extremely simple and has sublinear training time and space. We present rigorous theoretical analysis on the effect of using coarsening operations and provides useful guidance on the choice of coarsening methods. Interestingly, our theoretical analysis shows that coarsening can also be considered as a type of regularization and may improve the generalization.  Finally, empirical results on real world datasets show that, simply applying off-the-shelf coarsening methods, we can reduce the number of nodes by up to a factor of ten without causing a noticeable downgrade in classification accuracy. To sum up, this paper adds a new and simple technique in the toolbox for scaling up GNNs; from our theoretical analysis and empirical studies, it proves to be highly effective.

\section{Acknowledgments}
This work is supported by Shanghai Sailing Program Grant No. 18YF1401200, National Natural Science Foundation of China Grant No. 61802069, Shanghai Science and Technology Commission Grant No. 17JC1420200, and Science and Technology Commission of Shanghai Municipality Project Grant No. 19511120700.
 
\bibliographystyle{ACM-Reference-Format}
\bibliography{paper}

\input{appendix}
\end{document}

%% file: exp.tex
\begin{table*}[!htbp]
	\setlength{\abovecaptionskip}{0.2cm}
	\caption{Summary of results in terms of mean classification accuracy and standard deviation (in percent) over 20 runs on different datasets. The coarsening ratios of GCN and APPNP are $c=[0.7, 0.5, 0.3, 0.1]$ for each dataset respectively. The highest accuracy for each model in each column is highlighted in bold.}\label{tab:result}
	\centering
	\begin{tabular}{lcccccccccc}\toprule
		\multirow{2}*{\textbf{Method}}& \multicolumn{2}{c}{\textbf{Cora}} & \multicolumn{2}{c}{\textbf{Citeseer}}& \multicolumn{2}{c}{\textbf{Pubmed}}& \multicolumn{2}{c}{\textbf{Coauthor Physics}} &
		\multicolumn{2}{c}{\textbf{DBLP}} \\ \cmidrule(r){2-11}
		&5 &Fixed&5 &Fixed&5 &Fixed&5 &20&5 &20\\ \midrule
		
		% SGC & 68.4$\pm$3.7 &81.7$\pm$0.1 &58.3$\pm$4.0 &71.3$\pm$0.2& 66.1$\pm$6.3& 78.9$\pm$0.1& & &61.3$\pm$4.3 &72.7$\pm$1.6 \\
		%GAT &69.9$\pm$4.2 &83.0$\pm$0.7 & 57.6$\pm$3.1& 72.5$\pm$0.7&68.4$\pm$4.3&79.0$\pm$0.3 & 91.3$\pm$1.5&92.5$\pm$1.0 &66.2$\pm$5.1 &73.7$\pm$3.0 \\
		%GraphSAGE &65.7$\pm$4.7 & &57.7$\pm$3.6 & & 65.4$\pm$4.0& & & &59.9$\pm$4.1 &70.1$\pm$2.6 \\
		GCN & 67.5$\pm$4.8 & 81.5$\pm$0.6 &57.3$\pm$3.7  & 71.1$\pm$0.7 & 67.4$\pm$5.6 &\textbf{79.0}$\pm$0.6 &91.2$\pm$2.1  & 93.7$\pm$0.6 & 61.5$\pm$4.8&72.6$\pm$2.3 \\ 
		GCN (c=0.7) & 67.9$\pm$4.3 & 82.3$\pm$0.6 & 57.5$\pm$5.9&71.8$\pm$0.4  &68.3$\pm$5.2 &78.9$\pm$0.4 & 91.0$\pm$1.9&\textbf{93.8}$\pm$0.6 &61.4$\pm$5.0 &72.1$\pm$2.1 \\
		GCN (c=0.5) &  68.8$\pm$4.6 & \textbf{82.7}$\pm$0.5 & 57.7$\pm$5.3 & \textbf{72.0}$\pm$0.5 &\textbf{68.9}$\pm$4.4 & 78.5$\pm$0.3& \textbf{91.5}$\pm$2.0&93.7$\pm$0.7 &61.8$\pm$4.8 &72.7$\pm$2.0 \\
		GCN (c=0.3) &  \textbf{69.4}$\pm$4.5 & 81.7$\pm$0.5 &58.1$\pm$5.2 & 71.4$\pm$0.3 &68.7$\pm$4.2 & 78.4$\pm$0.4&90.8$\pm$2.3 &93.4$\pm$0.6&64.8$\pm$5.2  &74.5$\pm$1.9\\
		GCN (c=0.1) &  67.6$\pm$5.1 & 77.8$\pm$0.7 &\textbf{58.3}$\pm$6.3 & 71.1$\pm$0.4 &68.5$\pm$5.2 & 78.3$\pm$0.5&87.8$\pm$3.6 &91.5$\pm$1.4&\textbf{67.9}$\pm$5.6 &\textbf{76.0}$\pm$2.1 \\
		\midrule 
		APPNP & 72.8$\pm$3.8  &  83.3$\pm$0.5& 59.4$\pm$4.5 &  71.8$\pm$0.5 & 70.4$\pm$4.9& 80.1$\pm$0.2&92.0$\pm$1.6 &\textbf{94.0}$\pm$0.6 &\textbf{72.9}$\pm$4.2 &79.0$\pm$1.1 \\
		APPNP (c=0.7) &\textbf{73.9}$\pm$4.6  & \textbf{83.9}$\pm$0.8 &59.7$\pm$4.3 & 71.8$\pm$0.6 & 70.7$\pm$5.5 &\textbf{80.4}$\pm$0.3 &\textbf{92.3}$\pm$1.6 &93.7$\pm$0.8& 72.0$\pm$4.5&78.7$\pm$1.3\\
		APPNP (c=0.5) & 73.4$\pm$4.3 &  83.7$\pm$0.7&60.4$\pm$4.8 & \textbf{72.0}$\pm$0.5 &\textbf{71.2}$\pm$5.0 &79.6$\pm$0.3 &91.8$\pm$1.9   &93.9$\pm$0.5 & 72.3$\pm$4.0&79.1$\pm$1.2 \\
		APPNP (c=0.3) &73.1$\pm$3.5  &  82.5$\pm$0.6&\textbf{60.9}$\pm$5.7 & 71.6$\pm$0.4 &70.6$\pm$5.3 &78.4$\pm$0.7 & 91.7$\pm$1.5&93.6$\pm$0.6 &72.7$\pm$4.2 &\textbf{79.7}$\pm$1.0 \\
		APPNP (c=0.1) &70.8$\pm$4.9  &  80.2$\pm$0.8&60.7$\pm$5.8 & 71.8$\pm$0.5 &70.4$\pm$4.9 &77.3$\pm$0.5 & 88.6$\pm$3.3 &91.0$\pm$1.2 & 72.1$\pm$5.8 & 79.0$\pm$1.7 \\
		\bottomrule
	\end{tabular}
\end{table*}

\begin{figure*}[tbp]
	\setlength{\abovecaptionskip}{0cm}
	\setlength{\belowcaptionskip}{-0.2cm}
	\centering
	\includegraphics[width=0.96\textwidth]{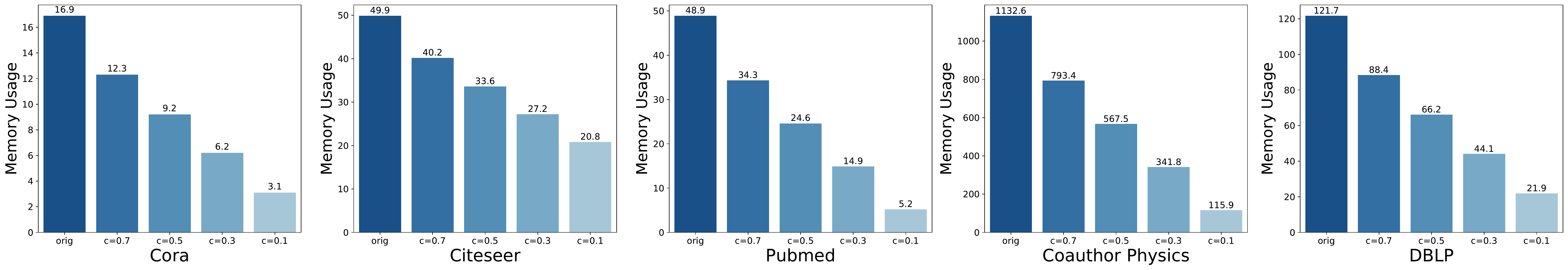}
	\caption{The Memory Usage of APPNP and coarse APPNP.}
	\label{fig:memory usage}
\end{figure*}

\section{Experiments} \label{sec:experiments}

In this section, we evaluate the performance of our method on two representative GNN architectures, namely GCN and APPNP: GCN has a structure with interlacing layers of graph diffusion and feature transformation, and APPNP decouples feature transformation from the diffusion. We compare the effect of different coarsening ratios on GCN and APPNP, including the  full-graph training. We also test the effect of several representative graph coarsening methods.

\subsection{Experimental Setup}

\textbf{Data splits.} The results are evaluated on five real world networks Cora, Citeseer, Pubmed, Coauthor Physics and DBLP \citep{kipf2016semi,bojchevski2018deep,shchur2018pitfalls} for semi-supervised node classification. Refer to the appendix for more details of the five datasets. For Cora, Citeseer, and Pubmed, we use the public split from \cite{yang2016revisiting}, which is widely used in the literature. In particular, the training set contains 20 labeled nodes per class, with an additional validation set of 500 and accuracy is evaluated on a test set of 1,000 nodes. 
For the other two datasets, the performance is tested on random splits \citep{shchur2018pitfalls}, where 20 labeled nodes per class are selected for training, 30 per class for validation, and all the other nodes are used for testing. Moreover, we also test the performances on each dataset under few label rates. We also evaluate in the few-shot regime, where, for each dataset, the training and validation set both have 5 labeled nodes per class, and the test set consists of all the rest. All the results are averaged over 20 runs and standard deviations are reported. 
%To evaluate the scalability of different coarsening methods, we also report their accuarcies and coarsening time.

\noindent\textbf{Implementation details.} For the original GCN and APPNP, we follow the settings suggested in the previous papers~\citep{shchur2018pitfalls,matthias2019pyg} for hyperparameters. In addition, we tuned the hyperparameter of models for better performance on Coauthor Physics and DBLP. For the fairness of comparison, our models use the same network architectures as baselines.  For evaluating the effect of different coarsening ratios, we report the results of variation neighborhoods coarsening; see \citep{loukas2019graph} for the detail. During the coarsening process, we remove super-nodes with mixed labels from the training set and the validation set, and also remove unlabeled isolated nodes. The detailed hyperparameter settings are listed in appendix.

\begin{table*}[!htbp]
	\setlength{\abovecaptionskip}{0.2cm}
	\caption{Summary of results in terms of accuracy, standard deviation and coarsening time(secs) with different coarsening methods.}\label{tab:coarsening methods}
	\centering
	\begin{tabular}{llccccccccc}\toprule
		\multirow{2}*{\textbf{Dataset}}&\multirow{2}*{\textbf{Coarsening Method}}& \multicolumn{3}{c}{\textbf{c=0.7}} & \multicolumn{3}{c}{\textbf{c=0.5}}& \multicolumn{3}{c}{\textbf{c=0.3}}\\ 
		\cmidrule(r){3-11}
		& &GCN &APPNP &Time&GCN &APPNP &Time&GCN &APPNP &Time \\ \midrule 
		\multirow{5}*{Cora}&Spectral Clustering & 82.2$\pm$0.5 & 83.2$\pm$0.4 &23.4 & 81.5$\pm$0.7 &  82.5$\pm$0.5& 16.3&79.4$\pm$0.5 &  78.0$\pm$1.3&10.0 \\
		&Variation Neighborhoods & \textbf{82.3}$\pm$0.6 & \textbf{83.9}$\pm$0.8 &2.0 & \textbf{82.7}$\pm$0.5 & 83.7$\pm$0.7 & 1.3& 81.7$\pm$0.5 &  \textbf{82.5}$\pm$0.6&2.1 \\
		&Variation Edges &\textbf{82.3}$\pm$0.5 &83.6$\pm$0.6 &0.3 &82.2$\pm$0.5 &\textbf{83.9}$\pm$0.5 &0.5 &80.0$\pm$0.4 & 81.1$\pm$0.7&0.6 \\
		&Algebraic JC &81.9$\pm$0.7  &82.9$\pm$0.7  &0.3 & 81.6$\pm$0.6 & 83.5$\pm$0.6 &0.5 & \textbf{82.2}$\pm$0.5 &\textbf{82.5}$\pm$0.7  &0.7 \\
		&Affinity GS & 81.4$\pm$0.4&83.3$\pm$0.4 &2.3 &82.0$\pm$0.7 &83.7$\pm$0.6 &3.2 &81.2$\pm$0.6 &81.9$\pm$1.1 &3.7 \\
		\midrule
		\multirow{5}*{DBLP}&Spectral Clustering & 71.5$\pm$2.2&78.9$\pm$1.0 &720.6 & 72.8$\pm$1.9& 78.7$\pm$0.9&492.2 &73.7$\pm$1.8 &77.4$\pm$1.3 & 273.5 \\
		&Variation Neighborhoods & 72.1$\pm$2.1 & 78.7$\pm$1.3  &8.3 & 72.7$\pm$2.0 &  79.1$\pm$1.2 & 9.4 & 74.5$\pm$1.9 &   79.7$\pm$1.0&12.6 \\
		&Variation Edges&  72.3$\pm$2.4  &78.9$\pm$1.0  &  2.8 &  73.4$\pm$1.9  &79.1$\pm$1.2 &  4.3  &  74.2$\pm$1.7  & 79.5$\pm$1.2  & 6.2\\
		&Algebraic JC & 72.5$\pm$2.3 &  78.6$\pm$1.6 &  3.0 &  73.1$\pm$2.0 & 78.3$\pm$1.1 & 5.5 &  74.0$\pm$1.7  &  79.1$\pm$1.2 &7.3\\
		&Affinity GS & \textbf{73.2}$\pm$2.1 &\textbf{79.2}$\pm$1.6 & 135.7 & \textbf{73.9}$\pm$1.7 &  \textbf{79.6}$\pm$0.7 & 199.6 & \textbf{75.3}$\pm$1.6 &   \textbf{79.9}$\pm$1.1& 225.9\\
		\bottomrule
	\end{tabular}
\end{table*}

\begin{figure*}[tbp]
	\setlength{\abovecaptionskip}{0.0cm}
	\setlength{\belowcaptionskip}{-0.4cm}
	\centering
	\subfigure[GCN]
	{\includegraphics[width=0.22\textwidth]{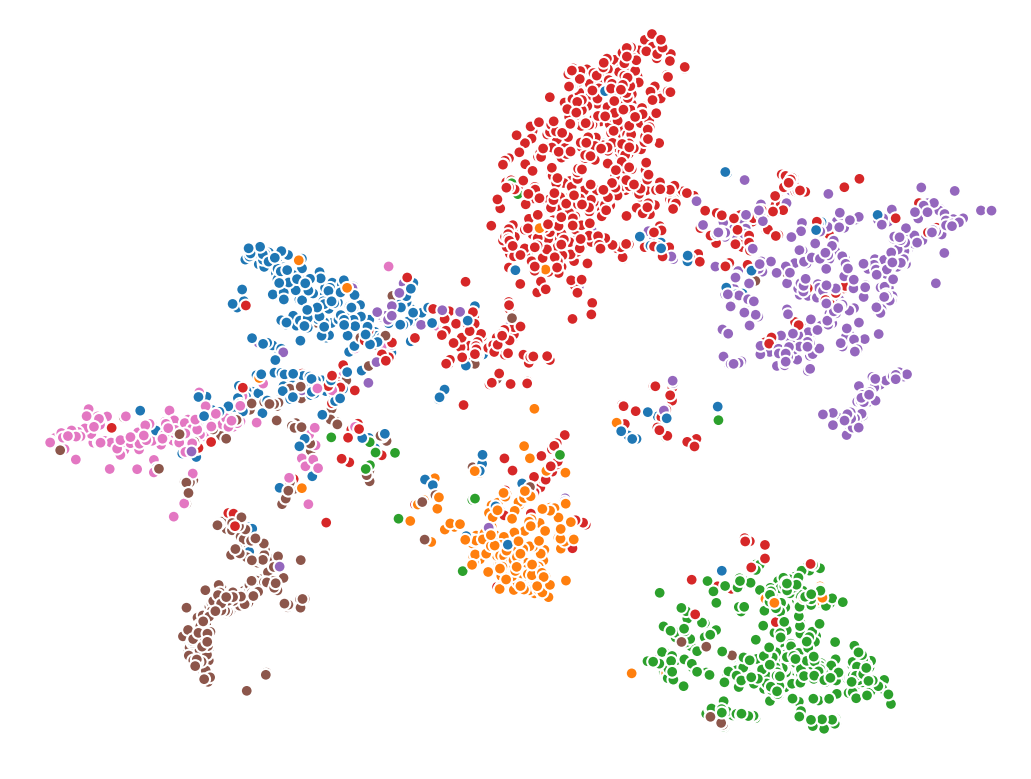}}
	\subfigure[GCN (c=0.7)]
	{\includegraphics[width=0.22\textwidth]{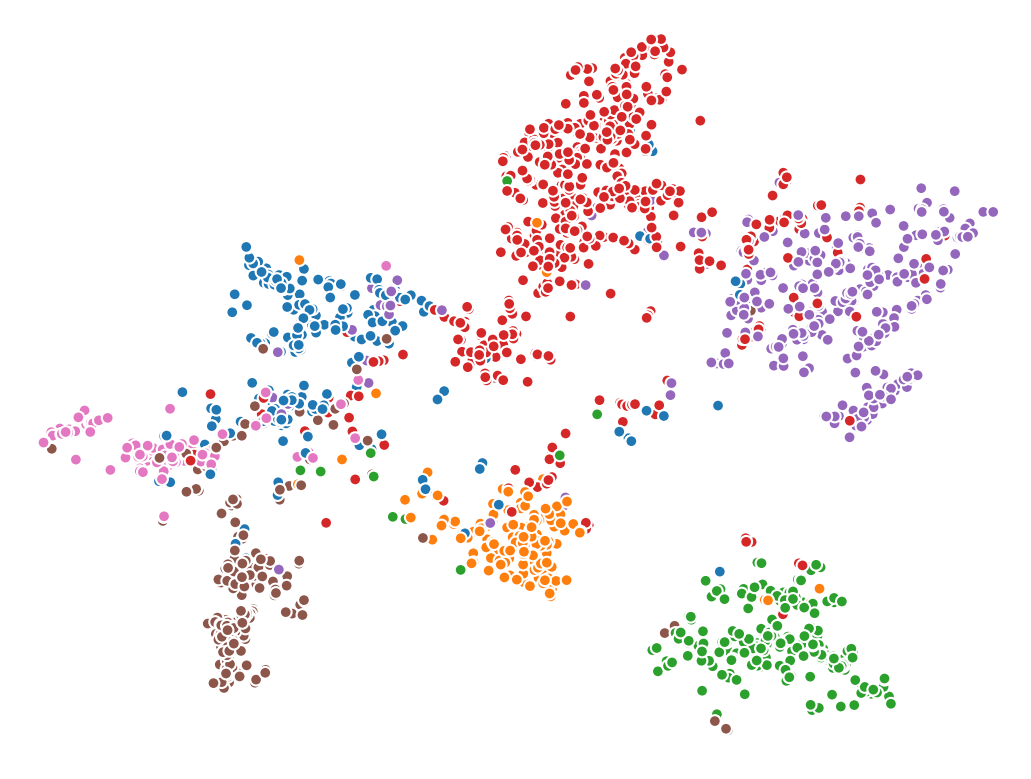}}
	\subfigure[GCN (c=0.5)]
	{\includegraphics[width=0.22\textwidth]{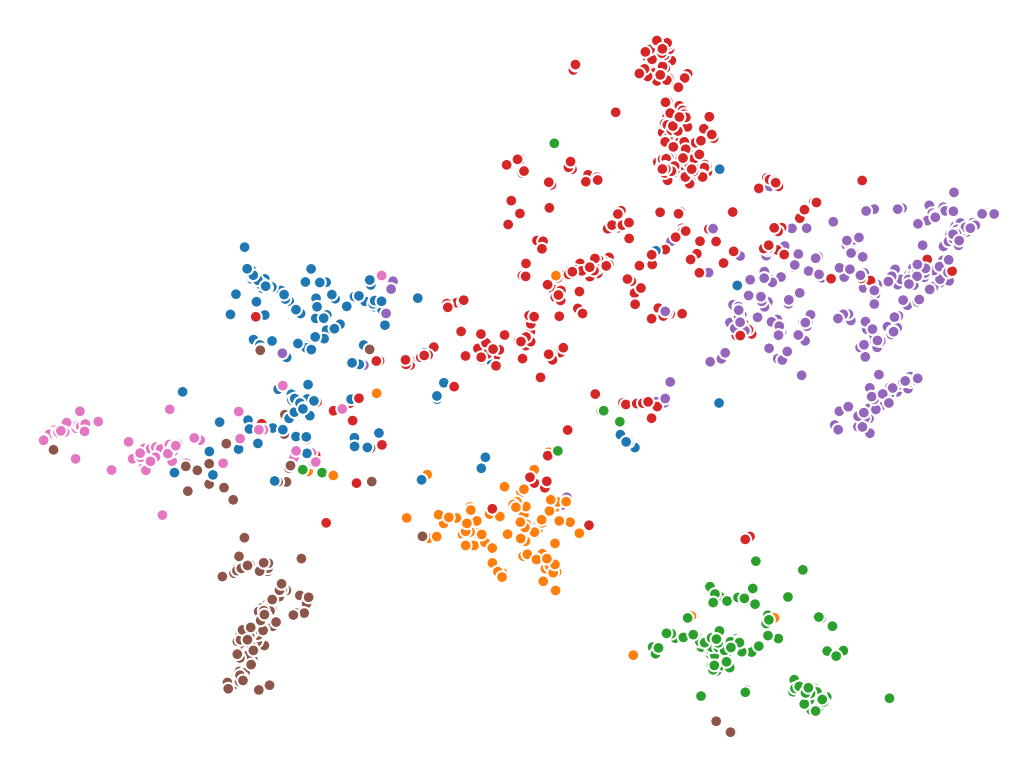}}
	\subfigure[GCN (c=0.3)]
	{\includegraphics[width=0.22\textwidth]{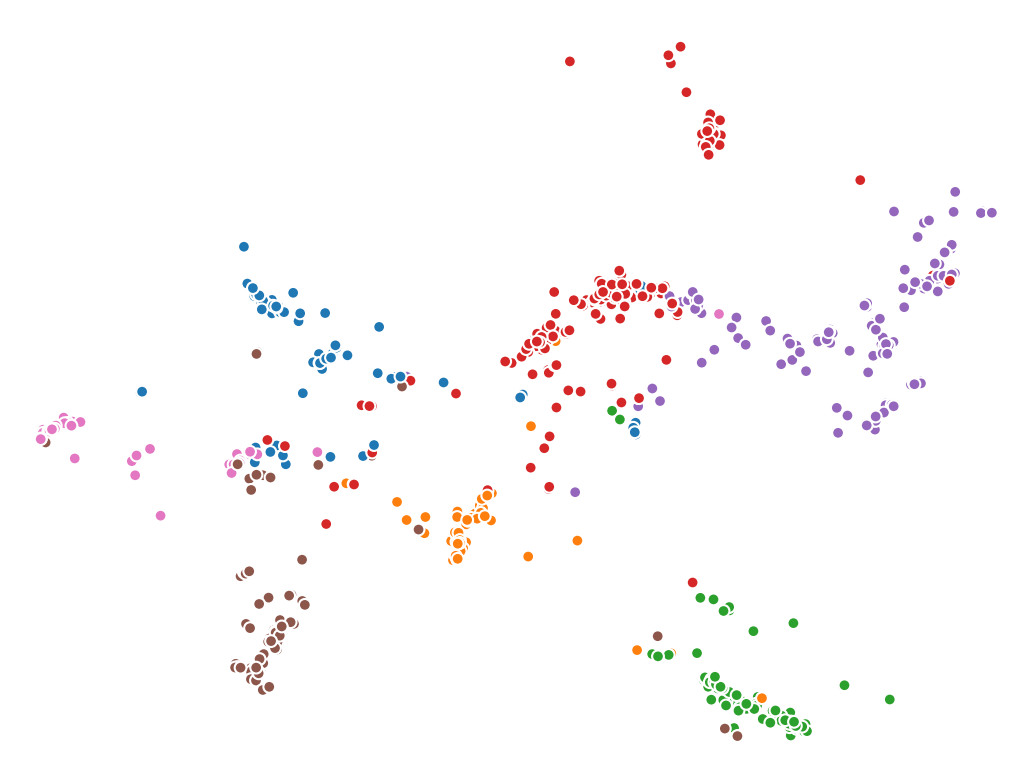}}
	\subfigure[APPNP]
	{\includegraphics[width=0.22\textwidth]{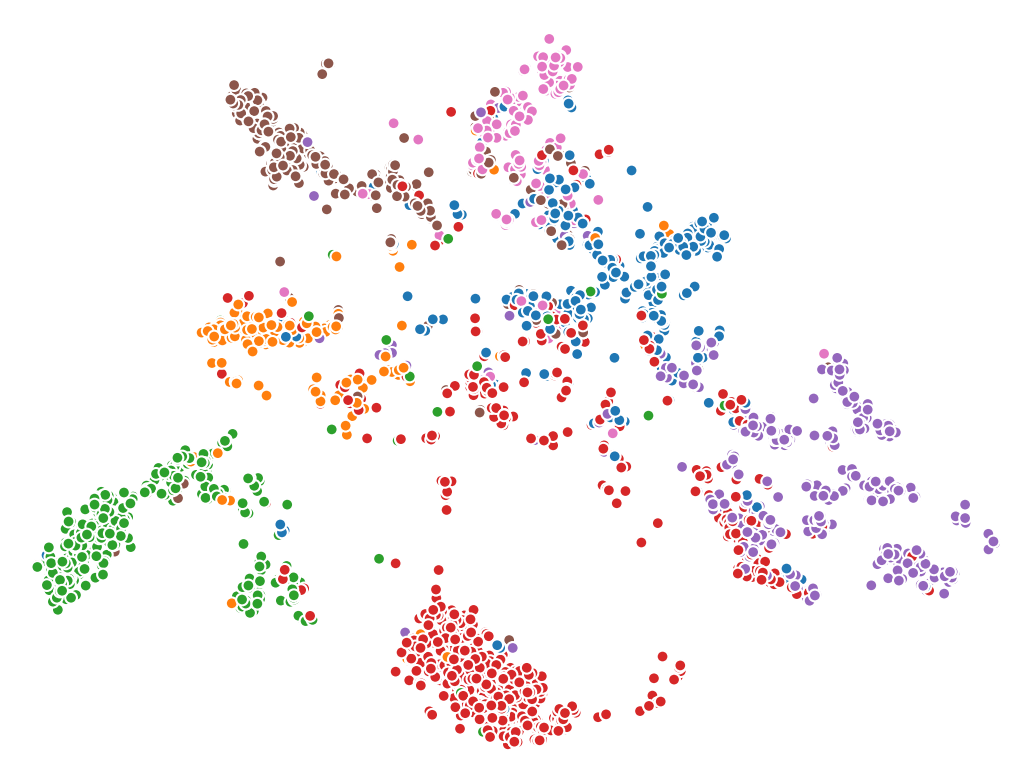}}
	\subfigure[APPNP (c=0.7)]
	{\includegraphics[width=0.22\textwidth]{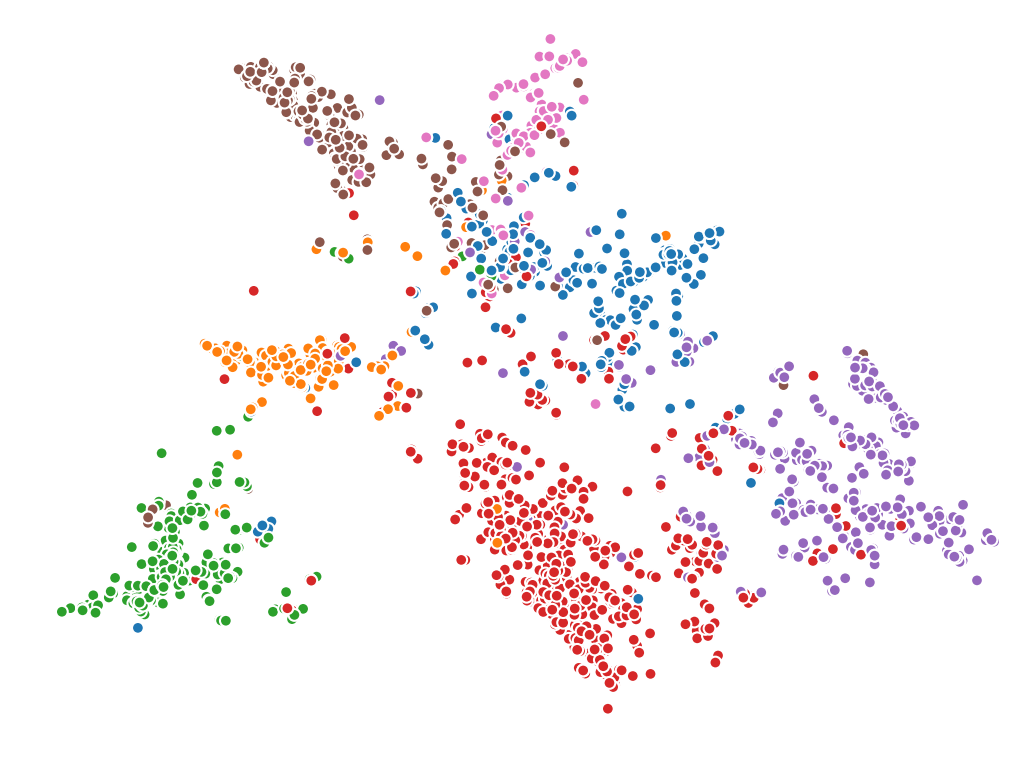}}
	\subfigure[APPNP (c=0.5)]
	{\includegraphics[width=0.22\textwidth]{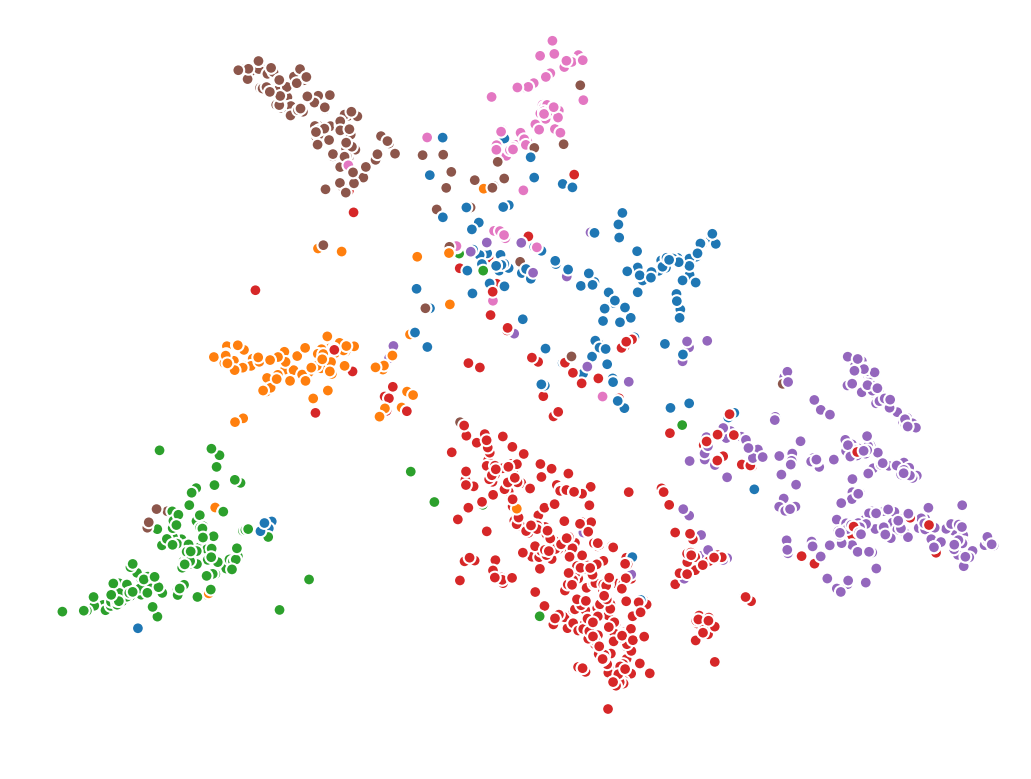}}
	\subfigure[APPNP (c=0.3)]
	{\includegraphics[width=0.22\textwidth]{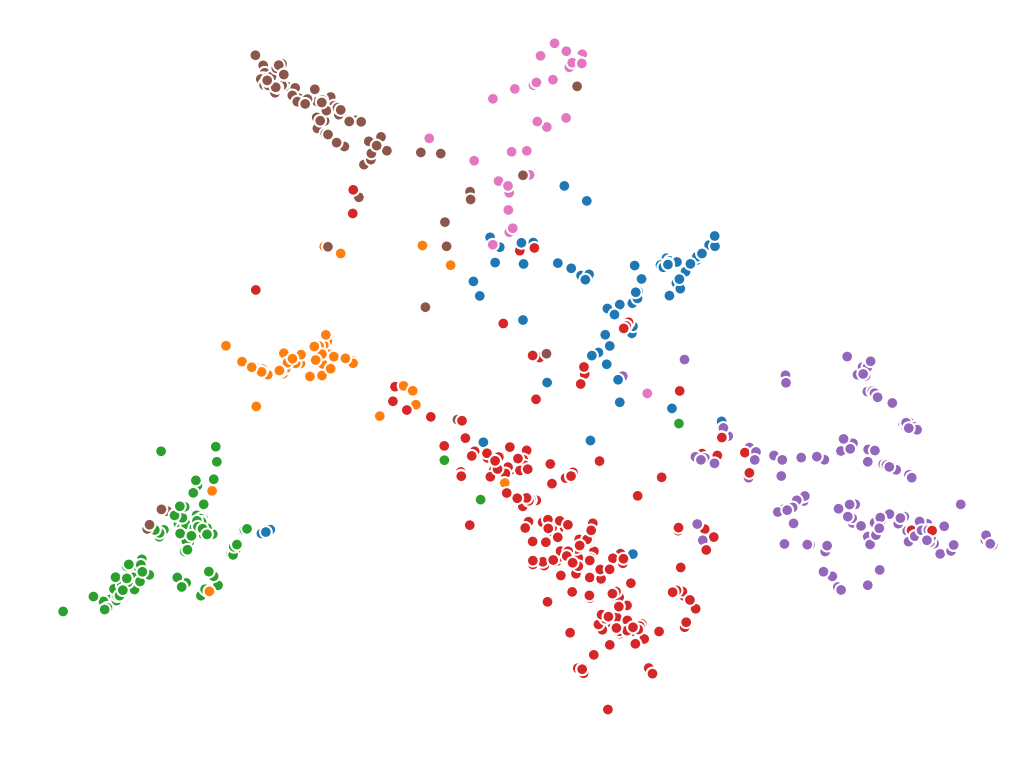}}
	\caption{Visualization of embeddings with t-SNE.}
	\label{fig:vis}
\end{figure*}

\subsection{Results and Analysis}

Table \ref{tab:result} presents the node classification accuracy and standard deviation of different coarsening ratios. The memory usages are summarized in Figure \ref{fig:memory usage}.

\noindent\textbf{Performance of GCN.} Our results demonstrate that coarse GCN achieves good performance across five datasets under diffenernt experimental settings. In most cases, the coarsening operation will not reduce the accuracy by much. Interestingly, the best result for all settings (except for the public split on Pubmed) is not achieved on full-graph training. This verifies our hypothesis on the regularization effect of graph coarsening. It is also observed that, when the coarsening ratio is 0.3, the performance of GCN is competitive against full-graph training; actually, the performance is improved on 7 out of 10 settings. Even when the graph is reduced by 10 times, the performance is still comparable and in 6 out of 10 cases, the accuracy is higher than or the same as using full-graph training. 

\noindent\textbf{Performance of APPNP.} For APPNP, we observe similar phenomenons as for GCN, even though the performance gain is not as noticeable as that on GCN. The resuts clearly are clearly consistent with our theoretical analysis.

\noindent\textbf{Memory Usage.} Figure \ref{fig:memory usage} shows the memory usage of APPNP with different coarsening ratios; The memory usages of GCN are very similar to APPNP, and thus we omit the results on GCN. Compared with the size of the input tensor, the space occupied by the parameters is very small, so the proportion of the space occupied by the coarse APPNP is close to the coarsening rate.
% \problemtext{(how to explain the memory usage of APPNP on citeseer)}

\noindent\textbf{Visualization.} We provide visualizations of the output layer with t-SNE for qualitative analyses. Here, we present the visualization results with different coarsening ratios on Cora in Figure \ref{fig:vis}, where nodes with the same color are from the same class. We clearly observe that, even though the number of nodes are different for each coarsening ratio, the overall distribution of node embeddings are quite similar across all ratios. This qualitatively verifies the theoretical analysis on the approximation quality of graph coarsening.

\subsection{Studies on Different Coarsening Methods}
Here we also study the efficacy of different coarsening methods for the proposed framework. We test the classification performance of four coarsening methods discussed in \cite{loukas2019graph} together with spectral clustering on Cora and DBLP. The four coarsening methods from \cite{loukas2019graph} are Variation Neighborhoods, Variation Edges, Algebraic JC and Affinity GS. In order to compare fairly, we use the same network structure and hyperparameters.

Table \ref{tab:coarsening methods} shows the result of different coarsening methods. Except for spectral clustering, there is no obvious difference between other coarsening methods. Compared with other methods, Variation Neighborhoods has best overall testing accuracies, and the coarsening time of variation neighborhoods is also acceptable. Variation Edge and Algebraic JC are competitive in classification accuracies, and their computational time is faster than Variation Neighborhoods.The time of spectral clustering is high mainly because the number of clusters in the $k$-means steps is large,and we can observe that the time goes down as the coarsening ratio gets lower.

%% file: appendix.tex
\newpage
\appendix
\section{Appendix}
Here we describe more details about the experiments to help in reproducibility.

\textbf{Datasets.} See Table \ref{tab:datasets} for a concise summary of the five datasets. The nodes in the networks are documents, each having a sparse bag-of-words feature vector; the edges represents citation links between documents.
\begin{table}[htbp]
	\setlength{\abovecaptionskip}{0.2cm}
	\setlength{\belowcaptionskip}{-0.2cm}
	\caption{Summary of the datasets used in our experiments}.	
	\begin{tabular}{lcccc}\toprule \centering
		\textbf{Dataset}&\textbf{Nodes}&\textbf{Edges}&\textbf{Features}&\textbf{Classes} \\
		\midrule
		Cora&2,708&5,429 &1,433&7 \\
		Citeseer&3,327&4,732&3,703&6  \\
		Pubmed&19,717&44,338&500&3  \\
		Coauthor Physics&34,493&247,962&8,415&5  \\
		DBLP&17,716&52,867&1,639&4  \\
		\bottomrule
		%\bottomrule
	\end{tabular}
	\label{tab:datasets}
	\centering
\end{table}

\textbf{Hyperparameters.} For the coarse GCN, we use Adam optimizer with learning rates of $[0.01, 0.01, 0.01, 0.001, 0.01]$ and a $L_2$ regularization with weights $[0.0005, 0.0005, 0.0005, 0, 0.0005]$. The number of training epochs are $[60, 200, 200, 200, 50]$ and the early stopping is set to $10$. For the coarse APPNP, $\alpha$ is set to $[0.1, 0.1, 0.1, 0.1, 0.05]$ and the number of layers is set to $[10, 10, 10, 20, 20]$ respectively. We use Adam optimizer with learning rates of $[0.01, 0.01, 0.01, 0.0005, 0.01]$ and a $L_2$ regularization with weights $[0.0005, 0.0005, 0.0005, 0, 0.0005]$. The number of training epochs are $[200, 200, 200, 500, 200]$ and the early stopping is set to $[10, 10, 10, 10, 0]$. The source code can be found in https://github.com/szzhang17/Scaling-Up-Graph-Neural-Networks-Via-Graph-Coarsening.

\textbf{Configuration.}
All the models are implemented in Python and PyTorch Geometric. Experiments are conducted on an NVIDIA 2080 Ti GPU, Intel(R) Core(TM) i7-10750H CPU@2.60GHz and Intel(R) Xeon(R) Silver 4116 CPU@2.10GHz. 

%% file: main.bbl

\begin{thebibliography}{47}

%%% the Web address.

\ifx \showCODEN    \undefined \def \showCODEN     #1{\unskip}     \fi
\ifx \showDOI      \undefined \def \showDOI       #1{#1}\fi
\ifx \showISBNx    \undefined \def \showISBNx     #1{\unskip}     \fi
\ifx \showISBNxiii \undefined \def \showISBNxiii  #1{\unskip}     \fi
\ifx \showISSN     \undefined \def \showISSN      #1{\unskip}     \fi
\ifx \showLCCN     \undefined \def \showLCCN      #1{\unskip}     \fi
\ifx \shownote     \undefined \def \shownote      #1{#1}          \fi
\ifx \showarticletitle \undefined \def \showarticletitle #1{#1}   \fi
\ifx \showURL      \undefined \def \showURL       {\relax}        \fi
% invisible to TeX
\providecommand\bibfield[2]{#2}
\providecommand\bibinfo[2]{#2}
\providecommand\natexlab[1]{#1}
\providecommand\showeprint[2][]{arXiv:#2}

\bibitem[\protect\citeauthoryear{Bojchevski and G{\"u}nnemann}{Bojchevski and
  G{\"u}nnemann}{2018}]%
        {bojchevski2018deep}
\bibfield{author}{\bibinfo{person}{Aleksandar Bojchevski} {and}
  \bibinfo{person}{Stephan G{\"u}nnemann}.} \bibinfo{year}{2018}\natexlab{}.
\newblock \showarticletitle{Deep Gaussian Embedding of Graphs: Unsupervised
  Inductive Learning via Ranking}. In \bibinfo{booktitle}{\emph{International
  Conference on Learning Representations}}.
\newblock


\bibitem[\protect\citeauthoryear{Bojchevski, Klicpera, Perozzi, Kapoor, Blais,
  R{\'o}zemberczki, Lukasik, and G{\"u}nnemann}{Bojchevski
  et~al\mbox{.}}{2020}]%
        {bojchevski2020scaling}
\bibfield{author}{\bibinfo{person}{Aleksandar Bojchevski},
  \bibinfo{person}{Johannes Klicpera}, \bibinfo{person}{Bryan Perozzi},
  \bibinfo{person}{Amol Kapoor}, \bibinfo{person}{Martin Blais},
  \bibinfo{person}{Benedek R{\'o}zemberczki}, \bibinfo{person}{Michal Lukasik},
  {and} \bibinfo{person}{Stephan G{\"u}nnemann}.}
  \bibinfo{year}{2020}\natexlab{}.
\newblock \showarticletitle{Scaling graph neural networks with approximate
  pagerank}. In \bibinfo{booktitle}{\emph{Proceedings of the 26th ACM SIGKDD
  International Conference on Knowledge Discovery \& Data Mining}}.
  \bibinfo{pages}{2464--2473}.
\newblock


\bibitem[\protect\citeauthoryear{Bruna, Zaremba, Szlam, and LeCun}{Bruna
  et~al\mbox{.}}{2014}]%
        {bruna2013spectral}
\bibfield{author}{\bibinfo{person}{Joan Bruna}, \bibinfo{person}{Wojciech
  Zaremba}, \bibinfo{person}{Arthur Szlam}, {and} \bibinfo{person}{Yann
  LeCun}.} \bibinfo{year}{2014}\natexlab{}.
\newblock \showarticletitle{Spectral networks and locally connected networks on
  graphs}. In \bibinfo{booktitle}{\emph{International Conference on Learning
  Representations}}.
\newblock


\bibitem[\protect\citeauthoryear{Chen, Ma, and Xiao}{Chen
  et~al\mbox{.}}{2018a}]%
        {chen2018fastgcn}
\bibfield{author}{\bibinfo{person}{Jie Chen}, \bibinfo{person}{Tengfei Ma},
  {and} \bibinfo{person}{Cao Xiao}.} \bibinfo{year}{2018}\natexlab{a}.
\newblock \showarticletitle{FastGCN: fast learning with graph convolutional
  networks via importance sampling}. In \bibinfo{booktitle}{\emph{International
  Conference on Learning Representations}}.
\newblock


\bibitem[\protect\citeauthoryear{Chen, Zhu, and Song}{Chen
  et~al\mbox{.}}{2018b}]%
        {chen2018stochastic}
\bibfield{author}{\bibinfo{person}{Jianfei Chen}, \bibinfo{person}{Jun Zhu},
  {and} \bibinfo{person}{Le Song}.} \bibinfo{year}{2018}\natexlab{b}.
\newblock \showarticletitle{Stochastic Training of Graph Convolutional Networks
  with Variance Reduction}. In \bibinfo{booktitle}{\emph{International
  Conference on Machine Learning}}. \bibinfo{pages}{942--950}.
\newblock


\bibitem[\protect\citeauthoryear{Chen, Wei, Ding, Li, Yuan, Du, and Wen}{Chen
  et~al\mbox{.}}{2020a}]%
        {chen2020scalable}
\bibfield{author}{\bibinfo{person}{Ming Chen}, \bibinfo{person}{Zhewei Wei},
  \bibinfo{person}{Bolin Ding}, \bibinfo{person}{Yaliang Li},
  \bibinfo{person}{Ye Yuan}, \bibinfo{person}{Xiaoyong Du}, {and}
  \bibinfo{person}{Ji-Rong Wen}.} \bibinfo{year}{2020}\natexlab{a}.
\newblock \showarticletitle{Scalable Graph Neural Networks via Bidirectional
  Propagation}. In \bibinfo{booktitle}{\emph{Advances in Neural Information
  Processing Systems}}.
\newblock


\bibitem[\protect\citeauthoryear{Chen, Wei, Huang, Ding, and Li}{Chen
  et~al\mbox{.}}{2020b}]%
        {chen2020simple}
\bibfield{author}{\bibinfo{person}{Ming Chen}, \bibinfo{person}{Zhewei Wei},
  \bibinfo{person}{Zengfeng Huang}, \bibinfo{person}{Bolin Ding}, {and}
  \bibinfo{person}{Yaliang Li}.} \bibinfo{year}{2020}\natexlab{b}.
\newblock \showarticletitle{Simple and deep graph convolutional networks}. In
  \bibinfo{booktitle}{\emph{International Conference on Machine Learning}}.
\newblock


\bibitem[\protect\citeauthoryear{Chiang, Liu, Si, Li, Bengio, and Hsieh}{Chiang
  et~al\mbox{.}}{2019}]%
        {chiang2019cluster}
\bibfield{author}{\bibinfo{person}{Wei-Lin Chiang}, \bibinfo{person}{Xuanqing
  Liu}, \bibinfo{person}{Si Si}, \bibinfo{person}{Yang Li},
  \bibinfo{person}{Samy Bengio}, {and} \bibinfo{person}{Cho-Jui Hsieh}.}
  \bibinfo{year}{2019}\natexlab{}.
\newblock \showarticletitle{Cluster-gcn: An efficient algorithm for training
  deep and large graph convolutional networks}. In
  \bibinfo{booktitle}{\emph{Proceedings of the 25th ACM SIGKDD International
  Conference on Knowledge Discovery \& Data Mining}}.
  \bibinfo{pages}{257--266}.
\newblock


\bibitem[\protect\citeauthoryear{Chung and Graham}{Chung and Graham}{1997}]%
        {chung1997spectral}
\bibfield{author}{\bibinfo{person}{Fan~RK Chung} {and}
  \bibinfo{person}{Fan~Chung Graham}.} \bibinfo{year}{1997}\natexlab{}.
\newblock \bibinfo{booktitle}{\emph{Spectral graph theory}}.
\newblock Number~92. \bibinfo{publisher}{American Mathematical Soc.}
\newblock


\bibitem[\protect\citeauthoryear{Cong, Forsati, Kandemir, and Mahdavi}{Cong
  et~al\mbox{.}}{2020}]%
        {cong2020minimal}
\bibfield{author}{\bibinfo{person}{Weilin Cong}, \bibinfo{person}{Rana
  Forsati}, \bibinfo{person}{Mahmut Kandemir}, {and} \bibinfo{person}{Mehrdad
  Mahdavi}.} \bibinfo{year}{2020}\natexlab{}.
\newblock \showarticletitle{Minimal variance sampling with provable guarantees
  for fast training of graph neural networks}. In
  \bibinfo{booktitle}{\emph{Proceedings of the 26th ACM SIGKDD International
  Conference on Knowledge Discovery \& Data Mining}}.
  \bibinfo{pages}{1393--1403}.
\newblock


\bibitem[\protect\citeauthoryear{Defferrard, Bresson, and
  Vandergheynst}{Defferrard et~al\mbox{.}}{2016}]%
        {defferrard2016convolutional}
\bibfield{author}{\bibinfo{person}{Micha{\"e}l Defferrard},
  \bibinfo{person}{Xavier Bresson}, {and} \bibinfo{person}{Pierre
  Vandergheynst}.} \bibinfo{year}{2016}\natexlab{}.
\newblock \showarticletitle{Convolutional neural networks on graphs with fast
  localized spectral filtering}. In \bibinfo{booktitle}{\emph{Advances in
  Neural Information Processing Systems}}. \bibinfo{pages}{3844--3852}.
\newblock


\bibitem[\protect\citeauthoryear{Deng, Zhao, Wang, Zhang, and Feng}{Deng
  et~al\mbox{.}}{2019}]%
        {deng2019graphzoom}
\bibfield{author}{\bibinfo{person}{Chenhui Deng}, \bibinfo{person}{Zhiqiang
  Zhao}, \bibinfo{person}{Yongyu Wang}, \bibinfo{person}{Zhiru Zhang}, {and}
  \bibinfo{person}{Zhuo Feng}.} \bibinfo{year}{2019}\natexlab{}.
\newblock \showarticletitle{GraphZoom: A Multi-level Spectral Approach for
  Accurate and Scalable Graph Embedding}. In
  \bibinfo{booktitle}{\emph{International Conference on Learning
  Representations}}.
\newblock


\bibitem[\protect\citeauthoryear{Dinella, Dai, Li, Naik, Song, and
  Wang}{Dinella et~al\mbox{.}}{2020}]%
        {dinella2020hoppity}
\bibfield{author}{\bibinfo{person}{Elizabeth Dinella}, \bibinfo{person}{Hanjun
  Dai}, \bibinfo{person}{Ziyang Li}, \bibinfo{person}{Mayur Naik},
  \bibinfo{person}{Le Song}, {and} \bibinfo{person}{Ke Wang}.}
  \bibinfo{year}{2020}\natexlab{}.
\newblock \showarticletitle{Hoppity: Learning graph transformations to detect
  and fix bugs in programs}. In \bibinfo{booktitle}{\emph{International
  Conference on Learning Representations}}.
\newblock


\bibitem[\protect\citeauthoryear{Eliav and Cohen}{Eliav and Cohen}{2018}]%
        {eliav2018bootstrapped}
\bibfield{author}{\bibinfo{person}{Buchnik Eliav} {and} \bibinfo{person}{Edith
  Cohen}.} \bibinfo{year}{2018}\natexlab{}.
\newblock \showarticletitle{Bootstrapped graph diffusions: Exposing the power
  of nonlinearity}. In \bibinfo{booktitle}{\emph{Proceedings of the ACM on
  Measurement and Analysis of Computing Systems}}.
\newblock


\bibitem[\protect\citeauthoryear{Englert, Gupta, Krauthgamer, Racke,
  Talgam-Cohen, and Talwar}{Englert et~al\mbox{.}}{2014}]%
        {englert2014vertex}
\bibfield{author}{\bibinfo{person}{Matthias Englert}, \bibinfo{person}{Anupam
  Gupta}, \bibinfo{person}{Robert Krauthgamer}, \bibinfo{person}{Harald Racke},
  \bibinfo{person}{Inbal Talgam-Cohen}, {and} \bibinfo{person}{Kunal Talwar}.}
  \bibinfo{year}{2014}\natexlab{}.
\newblock \showarticletitle{Vertex sparsifiers: New results from old
  techniques}.
\newblock \bibinfo{journal}{\emph{SIAM J. Comput.}} (\bibinfo{year}{2014}).
\newblock


\bibitem[\protect\citeauthoryear{Fahrbach, Goranci, Peng, Sachdeva, and
  Wang}{Fahrbach et~al\mbox{.}}{2020}]%
        {fahrbach2020faster}
\bibfield{author}{\bibinfo{person}{Matthew Fahrbach}, \bibinfo{person}{Gramoz
  Goranci}, \bibinfo{person}{Richard Peng}, \bibinfo{person}{Sushant Sachdeva},
  {and} \bibinfo{person}{Chi Wang}.} \bibinfo{year}{2020}\natexlab{}.
\newblock \showarticletitle{Faster graph embeddings via coarsening}. In
  \bibinfo{booktitle}{\emph{International Conference on Machine Learning}}.
\newblock


\bibitem[\protect\citeauthoryear{Hamilton, Ying, and Leskovec}{Hamilton
  et~al\mbox{.}}{2017}]%
        {hamilton2017inductive}
\bibfield{author}{\bibinfo{person}{Will Hamilton}, \bibinfo{person}{Zhitao
  Ying}, {and} \bibinfo{person}{Jure Leskovec}.}
  \bibinfo{year}{2017}\natexlab{}.
\newblock \showarticletitle{Inductive representation learning on large graphs}.
  In \bibinfo{booktitle}{\emph{Advances in Neural Information Processing
  Systems}}. \bibinfo{pages}{1024--1034}.
\newblock


\bibitem[\protect\citeauthoryear{Jin, Loukas, and JaJa}{Jin
  et~al\mbox{.}}{2020}]%
        {jin2020graph}
\bibfield{author}{\bibinfo{person}{Yu Jin}, \bibinfo{person}{Andreas Loukas},
  {and} \bibinfo{person}{Joseph JaJa}.} \bibinfo{year}{2020}\natexlab{}.
\newblock \showarticletitle{Graph coarsening with preserved spectral
  properties}. In \bibinfo{booktitle}{\emph{International Conference on
  Artificial Intelligence and Statistics}}.
\newblock


\bibitem[\protect\citeauthoryear{Kato}{Kato}{1995}]%
        {kato2013perturbation}
\bibfield{author}{\bibinfo{person}{Tosio Kato}.}
  \bibinfo{year}{1995}\natexlab{}.
\newblock \bibinfo{booktitle}{\emph{Perturbation theory for linear operators}}.
\newblock \bibinfo{publisher}{Springer Science \& Business Media}.
\newblock


\bibitem[\protect\citeauthoryear{Kipf and Welling}{Kipf and Welling}{2017}]%
        {kipf2016semi}
\bibfield{author}{\bibinfo{person}{Thomas~N Kipf} {and} \bibinfo{person}{Max
  Welling}.} \bibinfo{year}{2017}\natexlab{}.
\newblock \showarticletitle{Semi-supervised classification with graph
  convolutional networks}. In \bibinfo{booktitle}{\emph{International
  Conference on Learning Representations}}.
\newblock


\bibitem[\protect\citeauthoryear{Klicpera, Bojchevski, and
  G{\"u}nnemann}{Klicpera et~al\mbox{.}}{2019}]%
        {klicpera2018predict}
\bibfield{author}{\bibinfo{person}{Johannes Klicpera},
  \bibinfo{person}{Aleksandar Bojchevski}, {and} \bibinfo{person}{Stephan
  G{\"u}nnemann}.} \bibinfo{year}{2019}\natexlab{}.
\newblock \showarticletitle{Predict then Propagate: Graph Neural Networks meet
  Personalized PageRank}. In \bibinfo{booktitle}{\emph{International Conference
  on Learning Representations}}.
\newblock


\bibitem[\protect\citeauthoryear{Li and Goldwasser}{Li and Goldwasser}{2019}]%
        {li2019encoding}
\bibfield{author}{\bibinfo{person}{Chang Li} {and} \bibinfo{person}{Dan
  Goldwasser}.} \bibinfo{year}{2019}\natexlab{}.
\newblock \showarticletitle{Encoding social information with graph
  convolutional networks forpolitical perspective detection in news media}. In
  \bibinfo{booktitle}{\emph{Proceedings of the 57th Annual Meeting of the
  Association for Computational Linguistics}}. \bibinfo{pages}{2594--2604}.
\newblock


\bibitem[\protect\citeauthoryear{Li and Schild}{Li and Schild}{2018}]%
        {li2018spectral}
\bibfield{author}{\bibinfo{person}{Huan Li} {and} \bibinfo{person}{Aaron
  Schild}.} \bibinfo{year}{2018}\natexlab{}.
\newblock \showarticletitle{Spectral subspace sparsification}. In
  \bibinfo{booktitle}{\emph{2018 IEEE 59th Annual Symposium on Foundations of
  Computer Science}}.
\newblock


\bibitem[\protect\citeauthoryear{Liang, Gurukar, and Parthasarathy}{Liang
  et~al\mbox{.}}{2018}]%
        {liang2018mile}
\bibfield{author}{\bibinfo{person}{Jiongqian Liang}, \bibinfo{person}{Saket
  Gurukar}, {and} \bibinfo{person}{Srinivasan Parthasarathy}.}
  \bibinfo{year}{2018}\natexlab{}.
\newblock \showarticletitle{Mile: A multi-level framework for scalable graph
  embedding}.
\newblock \bibinfo{journal}{\emph{arXiv preprint arXiv:1802.09612}}
  (\bibinfo{year}{2018}).
\newblock


\bibitem[\protect\citeauthoryear{Liu, Gao, and Ji}{Liu et~al\mbox{.}}{2020}]%
        {liu2020towards}
\bibfield{author}{\bibinfo{person}{Meng Liu}, \bibinfo{person}{Hongyang Gao},
  {and} \bibinfo{person}{Shuiwang Ji}.} \bibinfo{year}{2020}\natexlab{}.
\newblock \showarticletitle{Towards deeper graph neural networks}. In
  \bibinfo{booktitle}{\emph{Proceedings of the 26th ACM SIGKDD International
  Conference on Knowledge Discovery \& Data Mining}}.
\newblock


\bibitem[\protect\citeauthoryear{Lloyd}{Lloyd}{1982}]%
        {lloyd1982least}
\bibfield{author}{\bibinfo{person}{Stuart Lloyd}.}
  \bibinfo{year}{1982}\natexlab{}.
\newblock \showarticletitle{Least squares quantization in PCM}.
\newblock \bibinfo{journal}{\emph{IEEE transactions on information theory}}
  \bibinfo{volume}{28}, \bibinfo{number}{2} (\bibinfo{year}{1982}),
  \bibinfo{pages}{129--137}.
\newblock


\bibitem[\protect\citeauthoryear{Loukas}{Loukas}{2019}]%
        {loukas2019graph}
\bibfield{author}{\bibinfo{person}{Andreas Loukas}.}
  \bibinfo{year}{2019}\natexlab{}.
\newblock \showarticletitle{Graph Reduction with Spectral and Cut Guarantees}.
\newblock \bibinfo{journal}{\emph{Journal of Machine Learning Research}}
  \bibinfo{volume}{20}, \bibinfo{number}{116} (\bibinfo{year}{2019}),
  \bibinfo{pages}{1--42}.
\newblock


\bibitem[\protect\citeauthoryear{Matthias~Fey}{Matthias~Fey}{2019}]%
        {matthias2019pyg}
\bibfield{author}{\bibinfo{person}{Jan E.~Lenssen Matthias~Fey}.}
  \bibinfo{year}{2019}\natexlab{}.
\newblock \showarticletitle{Fast Graph Representation Learning with PyTorch
  Geometric}. In \bibinfo{booktitle}{\emph{International Conference on Learning
  Representations Workshop}}.
\newblock


\bibitem[\protect\citeauthoryear{Moitra}{Moitra}{2009}]%
        {moitra2009approximation}
\bibfield{author}{\bibinfo{person}{Ankur Moitra}.}
  \bibinfo{year}{2009}\natexlab{}.
\newblock \showarticletitle{Approximation algorithms for multicommodity-type
  problems with guarantees independent of the graph size}. In
  \bibinfo{booktitle}{\emph{2009 50th Annual IEEE Symposium on Foundations of
  Computer Science}}.
\newblock


\bibitem[\protect\citeauthoryear{Monti, Boscaini, Masci, Rodola, Svoboda, and
  Bronstein}{Monti et~al\mbox{.}}{2017}]%
        {monti2017geometric}
\bibfield{author}{\bibinfo{person}{Federico Monti}, \bibinfo{person}{Davide
  Boscaini}, \bibinfo{person}{Jonathan Masci}, \bibinfo{person}{Emanuele
  Rodola}, \bibinfo{person}{Jan Svoboda}, {and} \bibinfo{person}{Michael~M
  Bronstein}.} \bibinfo{year}{2017}\natexlab{}.
\newblock \showarticletitle{Geometric deep learning on graphs and manifolds
  using mixture model CNNs}. In \bibinfo{booktitle}{\emph{2017 IEEE Conference
  on Computer Vision and Pattern Recognition}}.
\newblock


\bibitem[\protect\citeauthoryear{Paliwal, Gimeno, Nair, Li, Lubin, Kohli, and
  Vinyals}{Paliwal et~al\mbox{.}}{2020}]%
        {paliwal2019reinforced}
\bibfield{author}{\bibinfo{person}{Aditya Paliwal}, \bibinfo{person}{Felix
  Gimeno}, \bibinfo{person}{Vinod Nair}, \bibinfo{person}{Yujia Li},
  \bibinfo{person}{Miles Lubin}, \bibinfo{person}{Pushmeet Kohli}, {and}
  \bibinfo{person}{Oriol Vinyals}.} \bibinfo{year}{2020}\natexlab{}.
\newblock \showarticletitle{Reinforced genetic algorithm learning for
  optimizing computation graphs}. In \bibinfo{booktitle}{\emph{International
  Conference on Learning Representations}}.
\newblock


\bibitem[\protect\citeauthoryear{Pfaff, Fortunato, Sanchez-Gonzalez, and
  Battaglia}{Pfaff et~al\mbox{.}}{2021}]%
        {pfaff2020learning}
\bibfield{author}{\bibinfo{person}{Tobias Pfaff}, \bibinfo{person}{Meire
  Fortunato}, \bibinfo{person}{Alvaro Sanchez-Gonzalez}, {and}
  \bibinfo{person}{Peter~W Battaglia}.} \bibinfo{year}{2021}\natexlab{}.
\newblock \showarticletitle{Learning Mesh-Based Simulation with Graph
  Networks}. In \bibinfo{booktitle}{\emph{International Conference on Learning
  Representations}}.
\newblock


\bibitem[\protect\citeauthoryear{Ramezani, Cong, Mahdavi, Sivasubramaniam, and
  Kandemir}{Ramezani et~al\mbox{.}}{2020}]%
        {ramezani2020gcn}
\bibfield{author}{\bibinfo{person}{Morteza Ramezani}, \bibinfo{person}{Weilin
  Cong}, \bibinfo{person}{Mehrdad Mahdavi}, \bibinfo{person}{Anand
  Sivasubramaniam}, {and} \bibinfo{person}{Mahmut Kandemir}.}
  \bibinfo{year}{2020}\natexlab{}.
\newblock \showarticletitle{GCN meets GPU: Decoupling “When to Sample” from
  “How to Sample”}. In \bibinfo{booktitle}{\emph{Advances in Neural
  Information Processing Systems}}.
\newblock


\bibitem[\protect\citeauthoryear{Rong, Huang, Xu, and Huang}{Rong
  et~al\mbox{.}}{2019}]%
        {rong2019dropedge}
\bibfield{author}{\bibinfo{person}{Yu Rong}, \bibinfo{person}{Wenbing Huang},
  \bibinfo{person}{Tingyang Xu}, {and} \bibinfo{person}{Junzhou Huang}.}
  \bibinfo{year}{2019}\natexlab{}.
\newblock \showarticletitle{DropEdge: Towards Deep Graph Convolutional Networks
  on Node Classification}. In \bibinfo{booktitle}{\emph{International
  Conference on Learning Representations}}.
\newblock


\bibitem[\protect\citeauthoryear{Rossi, Frasca, Chamberlain, Eynard, Bronstein,
  and Monti}{Rossi et~al\mbox{.}}{2020}]%
        {rossi2020sign}
\bibfield{author}{\bibinfo{person}{Emanuele Rossi}, \bibinfo{person}{Fabrizio
  Frasca}, \bibinfo{person}{Ben Chamberlain}, \bibinfo{person}{Davide Eynard},
  \bibinfo{person}{Michael Bronstein}, {and} \bibinfo{person}{Federico Monti}.}
  \bibinfo{year}{2020}\natexlab{}.
\newblock \showarticletitle{Sign: Scalable inception graph neural networks}.
\newblock \bibinfo{journal}{\emph{arXiv preprint arXiv:2004.11198}}
  (\bibinfo{year}{2020}).
\newblock


\bibitem[\protect\citeauthoryear{Shchur, Mumme, Bojchevski, and
  G{\"{u}}nnemann}{Shchur et~al\mbox{.}}{2018}]%
        {shchur2018pitfalls}
\bibfield{author}{\bibinfo{person}{Oleksandr Shchur},
  \bibinfo{person}{Maximilian Mumme}, \bibinfo{person}{Aleksandar Bojchevski},
  {and} \bibinfo{person}{Stephan G{\"{u}}nnemann}.}
  \bibinfo{year}{2018}\natexlab{}.
\newblock \showarticletitle{Pitfalls of Graph Neural Network Evaluation}.
\newblock \bibinfo{journal}{\emph{arXiv preprint arXiv:1811.05868}}
  (\bibinfo{year}{2018}).
\newblock


\bibitem[\protect\citeauthoryear{Velickovic, Cucurull, Casanova, Romero, Lio,
  and Bengio}{Velickovic et~al\mbox{.}}{2018}]%
        {velickovic2017graph}
\bibfield{author}{\bibinfo{person}{Petar Velickovic}, \bibinfo{person}{Guillem
  Cucurull}, \bibinfo{person}{Arantxa Casanova}, \bibinfo{person}{Adriana
  Romero}, \bibinfo{person}{Pietro Lio}, {and} \bibinfo{person}{Yoshua
  Bengio}.} \bibinfo{year}{2018}\natexlab{}.
\newblock \showarticletitle{Graph attention networks}. In
  \bibinfo{booktitle}{\emph{International Conference on Learning
  Representations}}.
\newblock


\bibitem[\protect\citeauthoryear{Wei, Goyal, Durrett, and Dillig}{Wei
  et~al\mbox{.}}{2020}]%
        {wei2020lambdanet}
\bibfield{author}{\bibinfo{person}{Jiayi Wei}, \bibinfo{person}{Maruth Goyal},
  \bibinfo{person}{Greg Durrett}, {and} \bibinfo{person}{Isil Dillig}.}
  \bibinfo{year}{2020}\natexlab{}.
\newblock \showarticletitle{Lambdanet: Probabilistic type inference using graph
  neural networks}. In \bibinfo{booktitle}{\emph{International Conference on
  Learning Representations}}.
\newblock


\bibitem[\protect\citeauthoryear{Wu, Souza, Zhang, Fifty, Yu, and
  Weinberger}{Wu et~al\mbox{.}}{2019}]%
        {wu2019simplifying}
\bibfield{author}{\bibinfo{person}{Felix Wu}, \bibinfo{person}{Amauri Souza},
  \bibinfo{person}{Tianyi Zhang}, \bibinfo{person}{Christopher Fifty},
  \bibinfo{person}{Tao Yu}, {and} \bibinfo{person}{Kilian Weinberger}.}
  \bibinfo{year}{2019}\natexlab{}.
\newblock \showarticletitle{Simplifying Graph Convolutional Networks}. In
  \bibinfo{booktitle}{\emph{International Conference on Machine Learning}}.
  \bibinfo{pages}{6861--6871}.
\newblock


\bibitem[\protect\citeauthoryear{Xu, Hu, Leskovec, and Jegelka}{Xu
  et~al\mbox{.}}{2018}]%
        {xu2018powerful}
\bibfield{author}{\bibinfo{person}{Keyulu Xu}, \bibinfo{person}{Weihua Hu},
  \bibinfo{person}{Jure Leskovec}, {and} \bibinfo{person}{Stefanie Jegelka}.}
  \bibinfo{year}{2018}\natexlab{}.
\newblock \showarticletitle{How Powerful are Graph Neural Networks?}. In
  \bibinfo{booktitle}{\emph{International Conference on Learning
  Representations}}.
\newblock


\bibitem[\protect\citeauthoryear{Yang, Cohen, and Salakhudinov}{Yang
  et~al\mbox{.}}{2016}]%
        {yang2016revisiting}
\bibfield{author}{\bibinfo{person}{Zhilin Yang}, \bibinfo{person}{William
  Cohen}, {and} \bibinfo{person}{Ruslan Salakhudinov}.}
  \bibinfo{year}{2016}\natexlab{}.
\newblock \showarticletitle{Revisiting Semi-Supervised Learning with Graph
  Embeddings}. In \bibinfo{booktitle}{\emph{International Conference on Machine
  Learning}}. \bibinfo{pages}{40--48}.
\newblock


\bibitem[\protect\citeauthoryear{Ying, He, Chen, Eksombatchai, Hamilton, and
  Leskovec}{Ying et~al\mbox{.}}{2018}]%
        {ying2018graph}
\bibfield{author}{\bibinfo{person}{Rex Ying}, \bibinfo{person}{Ruining He},
  \bibinfo{person}{Kaifeng Chen}, \bibinfo{person}{Pong Eksombatchai},
  \bibinfo{person}{William~L Hamilton}, {and} \bibinfo{person}{Jure Leskovec}.}
  \bibinfo{year}{2018}\natexlab{}.
\newblock \showarticletitle{Graph convolutional neural networks for web-scale
  recommender systems}. In \bibinfo{booktitle}{\emph{Proceedings of the 24th
  ACM SIGKDD International Conference on Knowledge Discovery \& Data Mining}}.
  \bibinfo{pages}{974--983}.
\newblock


\bibitem[\protect\citeauthoryear{Zeng, Zhou, Srivastava, Kannan, and
  Prasanna}{Zeng et~al\mbox{.}}{2019}]%
        {zeng2019graphsaint}
\bibfield{author}{\bibinfo{person}{Hanqing Zeng}, \bibinfo{person}{Hongkuan
  Zhou}, \bibinfo{person}{Ajitesh Srivastava}, \bibinfo{person}{Rajgopal
  Kannan}, {and} \bibinfo{person}{Viktor Prasanna}.}
  \bibinfo{year}{2019}\natexlab{}.
\newblock \showarticletitle{GraphSAINT: Graph Sampling Based Inductive Learning
  Method}. In \bibinfo{booktitle}{\emph{International Conference on Learning
  Representations}}.
\newblock


\bibitem[\protect\citeauthoryear{Zhang and Chen}{Zhang and Chen}{2019}]%
        {zhang2019inductive}
\bibfield{author}{\bibinfo{person}{Muhan Zhang} {and} \bibinfo{person}{Yixin
  Chen}.} \bibinfo{year}{2019}\natexlab{}.
\newblock \showarticletitle{Inductive Matrix Completion Based on Graph Neural
  Networks}. In \bibinfo{booktitle}{\emph{International Conference on Learning
  Representations}}.
\newblock


\bibitem[\protect\citeauthoryear{Zhou, Bousquet, Lal, Weston, and
  Sch{\"o}lkopf}{Zhou et~al\mbox{.}}{2004}]%
        {zhou2004learning}
\bibfield{author}{\bibinfo{person}{Dengyong Zhou}, \bibinfo{person}{Olivier
  Bousquet}, \bibinfo{person}{Thomas~Navin Lal}, \bibinfo{person}{Jason
  Weston}, {and} \bibinfo{person}{Bernhard Sch{\"o}lkopf}.}
  \bibinfo{year}{2004}\natexlab{}.
\newblock \showarticletitle{Learning with local and global consistency}. In
  \bibinfo{booktitle}{\emph{Advances in neural information processing
  systems}}.
\newblock


\bibitem[\protect\citeauthoryear{Zhu, Wang, Shi, Ji, and Cui}{Zhu
  et~al\mbox{.}}{2021}]%
        {zhu2021interpreting}
\bibfield{author}{\bibinfo{person}{Meiqi Zhu}, \bibinfo{person}{Xiao Wang},
  \bibinfo{person}{Chuan Shi}, \bibinfo{person}{Houye Ji}, {and}
  \bibinfo{person}{Peng Cui}.} \bibinfo{year}{2021}\natexlab{}.
\newblock \showarticletitle{Interpreting and Unifying Graph Neural Networks
  with An Optimization Framework}. In \bibinfo{booktitle}{\emph{Proceedings of
  The Web Conference 2021}}.
\newblock


\bibitem[\protect\citeauthoryear{Zou, Hu, Wang, Jiang, Sun, and Gu}{Zou
  et~al\mbox{.}}{2019}]%
        {zou2019layer}
\bibfield{author}{\bibinfo{person}{Difan Zou}, \bibinfo{person}{Ziniu Hu},
  \bibinfo{person}{Yewen Wang}, \bibinfo{person}{Song Jiang},
  \bibinfo{person}{Yizhou Sun}, {and} \bibinfo{person}{Quanquan Gu}.}
  \bibinfo{year}{2019}\natexlab{}.
\newblock \showarticletitle{Layer-Dependent Importance Sampling for Training
  Deep and Large Graph Convolutional Networks}. In
  \bibinfo{booktitle}{\emph{Advances in neural information processing
  systems}}.
\newblock


\end{thebibliography}
